\documentclass{article}

\usepackage{microtype}
\usepackage{graphicx}
\usepackage{booktabs} 

\usepackage{hyperref}

\usepackage[accepted]{icml2025}

\usepackage{amsmath}
\usepackage{amssymb}
\usepackage{mathtools}
\usepackage{amsthm}

\usepackage[capitalize,noabbrev]{cleveref}

\usepackage{balance} 
\usepackage{multirow} 

\usepackage{algpseudocode}
\usepackage{amsmath}
\usepackage{subcaption}
\usepackage{amsthm}

\newcommand{\argmax}{\operatornamewithlimits{argmax}}

\newcommand{\var}{\text{var}}

\newcommand{\squishlist}{\begin{list}{$\bullet$}{\topsep=1pt \parsep=0pt \itemsep=1pt \leftmargin=1em }} 
\newcommand{\squishend}{\end{list}}

\newcommand{\beitemize}{\begin{list}{$\bullet$}{}} 
\newcommand{\enitemize}{\end{list}}

\theoremstyle{plain}
\newtheorem{theorem}{Theorem}[section]
\newtheorem{proposition}[theorem]{Proposition}

\newtheorem{corollary}[theorem]{Corollary}
\theoremstyle{definition}

\theoremstyle{remark}

\usepackage[textsize=tiny]{todonotes}

\icmltitlerunning{DECAF: Learning to be Fair in Multi-agent Resource Allocation}

\begin{document}

\twocolumn[
\icmltitle{DECAF: Learning to be Fair in Multi-agent Resource Allocation}



\icmlsetsymbol{equal}{*}

\begin{icmlauthorlist}
\icmlauthor{Ashwin Kumar}{sch}
\icmlauthor{William Yeoh}{sch}
\end{icmlauthorlist}

\icmlaffiliation{sch}{Washington University in St Louis, USA}

\icmlcorrespondingauthor{Ashwin Kumar}{ashwinkumar@wustl.edu}
\icmlcorrespondingauthor{William Yeoh}{wyeoh@wustl.edu}

\icmlkeywords{Reinforcement Learning, ICML, Multiagent Systems, Reosurce Allocation, Fairness}

\vskip 0.3in
]




\begin{abstract}
A wide variety of resource allocation problems operate under resource constraints that are managed by a central arbitrator, with agents who evaluate and communicate preferences over these resources. We formulate this broad class of problems as  \emph{Distributed Evaluation, Centralized Allocation (DECA)} problems and propose methods to learn fair and efficient policies in centralized resource allocation. Our methods are applied to learning long-term fairness in a novel and general framework for fairness in multi-agent systems. We show three different methods based on Double Deep Q-Learning: (1) A joint weighted optimization of fairness and utility, (2) a split optimization, learning two separate Q-estimators for utility and fairness, and (3) an online policy perturbation to guide existing black-box utility functions toward fair solutions. Our methods outperform existing fair MARL approaches on multiple resource allocation domains, even when evaluated using diverse fairness functions, and allow for flexible online trade-offs between utility and fairness.
\end{abstract}

\section{Introduction}

\label{sec:intro}

AI has become an essential component of many modern systems, playing a crucial role in automating complex decision-making processes. Increasingly, AI algorithms are used to make decisions that impact millions of people. In multi-agent settings, these decisions are often optimized for overall system utility. However, this utilitarian approach can introduce biases, making fairness an important consideration in AI-driven decision-making.

We introduce a class of problems called \emph{Distributed Evaluation, Centralized Allocation (DECA)}. To the best of our knowledge, prior work has addressed these problems separately using domain-specific solutions~\citep{shah2020neural, qin2022RLforRides,kube2019allocating, kube2023community}. In this paper, we take a first step toward unifying them under a single DECA framework and propose fairness-oriented methods that apply broadly across DECA problems.

In DECA, multiple agents act within an environment while a central controller coordinates their behavior to ensure resource constraints and environmental requirements are met. Each agent evaluates its own actions (Distributed Evaluation, DE), and the central controller aggregates these evaluations, optimizes for system-wide utility, and assigns actions accordingly (Centralized Allocation, CA). These problems are dynamic in nature, 
with time-varying resources and agents with variable action spaces. 
This makes DECA both computationally challenging and practically significant.\footnote{
Note that DECA is an execution paradigm, not a learning paradigm like Centralized Training/Decentralized Execution (CTDE). CTDE can be used to train agents that operate in a DECA environment, as we do in this work.
}

Fairness in DECA-based decision-making is critical, as algorithmic biases can lead to disparities and decreased trust in automated systems~\citep{mehrabi2021MLBiassurvey}. Beyond ethical considerations, fair resource allocation may also be desirable from the perspective of the central controller. Standard DECA solutions rely on estimating agent utilities and solving a constrained optimization to compute the best action for each agent. To improve fairness, we propose a framework that integrates fairness estimation into DECA, allowing for more balanced allocations. Specifically, we introduce three optimization strategies:

\squishlist
\item \textbf{Joint Optimization (JO)}: A scalarized multi-objective learning approach that jointly optimizes for fairness and utility.
\item \textbf{Split Optimization (SO)}: A method that learns separate fairness and utility estimators, enabling online trade-off adjustments for fairness and utility.
\item \textbf{Fair-Only Optimization (FO)}: A fairness-focused approach that modifies an existing black-box utility function to incorporate fairness considerations.
\squishend

Our approach is broadly applicable across different domains and fairness metrics, as demonstrated through empirical evaluations. We compare these methods and show how each offers unique advantages in different scenarios, using variance in agent utilities as a fairness metric. Notably, the Split and Fair-Only Optimization approaches enable real-time tuning of the fairness-utility trade-off, an important consideration for real-world applications that has been overlooked in prior work on fairness in multi-agent RL.

This work addresses a critical gap in the literature by providing a unified framework for fairness in DECA problems, which encompass a wide range of multi-agent decision-making scenarios. Through our proposed optimization strategies (JO, SO, and FO), we offer a flexible framework for balancing fairness and efficiency in real time. Furthermore, we present the first general approach for integrating fairness into multi-agent resource allocation using Q-learning, paving the way for future advancements in fair AI decision-making.

\section{Related Work}
\label{sec:related}

While DECA has not been formalized in prior work, it has seen application in many domains. From optimizing passenger-driver matches in ridesharing~\citep{shah2020neural, qin2022RLforRides} to efficient allocation of homelessness resources~\citep{kube2019allocating, kube2023community}, many real-world applications follow this general structure. 
It is also analogous to the predict-then-optimize (P+O) approach \citep{wang2021PtO, elmachtoub2022smartPtO}, where a predictive model estimates unknown parameters that are subsequently used in optimization. However, unlike P+O, which may not specifically address resource allocation or multi-agent systems, DECA explicitly focuses on these complexities. This distinction is crucial as it allows us to restrict the problem space and tailor our research towards enhancing fairness within multi-agent resource allocation.

Significant research has addressed algorithmic bias, where ML models, such as those used in hiring decisions \citep{raghavan2020hiringbias}, can exhibit harmful biases. 
We refer readers to an extensive survey by \citet{mehrabi2021MLBiassurvey} for a review of recent work.
These studies typically focus on debiasing the outputs of predictive models to meet fairness criteria such as equalized odds~\citep{hardt2016equality} or demographic parity~\citep{DP_dwork2012}. However, our work diverges from this approach. Instead of correcting biases in predictions, we aim to develop algorithms that inherently promote fair decision-making via the actions they optimize.

Our focus in this paper is on designing ways to learn fair policies in a multi-agent RL setting. \citet{FairRLSurvey} present a survey on RL methods used to improve fairness. We now highlight a few papers that are the closest to our work: FEN~\citep{jiang2019FEN} uses a hierarchical network to learn a fair-efficient policy for multi-agent coordination, learning to optimize the coefficient of variation, with a meta network that selects when each agent behaves greedily or fairly. However, the model does not allow for resource constraints, instead opting for a first-come-first-serve approach. Further, this approach needs communication between agents to allow agents to choose between acting fairly and efficiently. Some methods, on the other hand, propose to optimize fairness in a multi-objective MDP, where each agent's utility is treated as a separate objective, and the goal is to optimize the a social welfare function over agent rewards~\citep{zimmer2021MOMDP,siddique2020MOMDP}. 
This means the learning agent has to predict the utility over the joint action space~\citep{siddique2020MOMDP}, or, as done by SOTO~\citep{zimmer2021MOMDP}, use a decentralized policy gradient based approach, which prevents use of global constraints. 
Finally, SI~\citep{SI_kumar2023} is an approach for improving fairness in rideshare-matching that attempts to improve fairness through myopic fairness post-processing of black-box utility estimates. However, it does not attempt to learn long-term fairness, and is specially designed for the ridesharing domain.

The DECA approach allows us to consider global constraints while allocating resources, opening up the scope for better global solutions, which none of the prior approaches allow. The distributed evaluation allows each agent to only learn a local value function, which reduces the complexity when compared to learning a joint policy. Further, our Split and Fair-Only approaches allow changing the trade-offs between utility and fairness post-training, which provides additional flexibility that previous approaches lack.

\section{Problem Formulation}
\label{sec:problem_formulation}
In the context of Distributed Evaluation, Centralized Allocation (DECA), our primary goal is to integrate fairness into the decision-making process of resource allocation in multi-agent systems. Formally, we seek to maximize a combined measure of system utility and fairness, represented as:
\begin{align}
    \max \, (1-\beta) \mathcal{U}_T + \beta \mathcal{F}_T \label{eq:objective}
\end{align}
where $\mathcal{U}_T$ denotes the total utility at time $T$ and $\mathcal{F}_T$ represents the fairness measure, weighted by~$\beta$. 

In this section, we describe the DECA optimization framework and how it can be used for resource allocation. In the next section, we will describe DECAF, our solution to learn to improve fairness in this framework.

\subsection{Distributed Evaluation, Centralized Allocation (DECA)}

We define the DECA framework through a temporal resource allocation problem formulated as a Constrained Multi-agent MDP~\citep{CMMDP_de2021}. Our model is described by the tuple $\mathcal{M}$ with the following components:
\begin{align}
\mathcal{M} = \langle \alpha, S, \mathcal{O}, \{A_i\}_{i \in \alpha}, T, \{R_u\}_{i \in \alpha}, \gamma, c \rangle 
\end{align}
\squishlist
    \item $\alpha$ is the set of agents indexed by $i$ ($n$ agents).
    \item $S$ is the global state space.
    \item $\mathcal{O}: S \rightarrow O_1 \times O_2 \times \ldots \times O_n$ is the observation function that maps the true state to agent observations.
    \item $A_i$ is the action space for agent $i$.
    \item $T: S \times A_1 \times A_2 \times \ldots \times A_n \times S \rightarrow [0,1]$ represents the joint transition probabilities.
    \item $R_u: S \times A_1 \times A_2 \times \ldots \times A_n  \rightarrow \mathbb{R}^n$ denotes the (utility) reward function, which returns a vector of rewards, one for each agent.
    \item $\gamma$ is the discount factor for future rewards.
    \item $c: A_1 \cup A_2 \cup \ldots \cup A_n \rightarrow \mathbb{R}^K$ maps each action to its resource consumption for $K$ types of resources.
\squishend

In a DECA problem, illustrated by Figure~\ref{fig:deca}, agents independently evaluate actions based on their local observations (DE), while a central controller aggregates these evaluations and optimizes resource allocation subject to constraints (CA). 
Agent actions include the null action and may have other effects apart from allocation of resources (e.g. moving in an environment), but only actions which consume resources are constrained. 

\begin{figure}[t]
    \centering
    \includegraphics[width=0.8\linewidth]{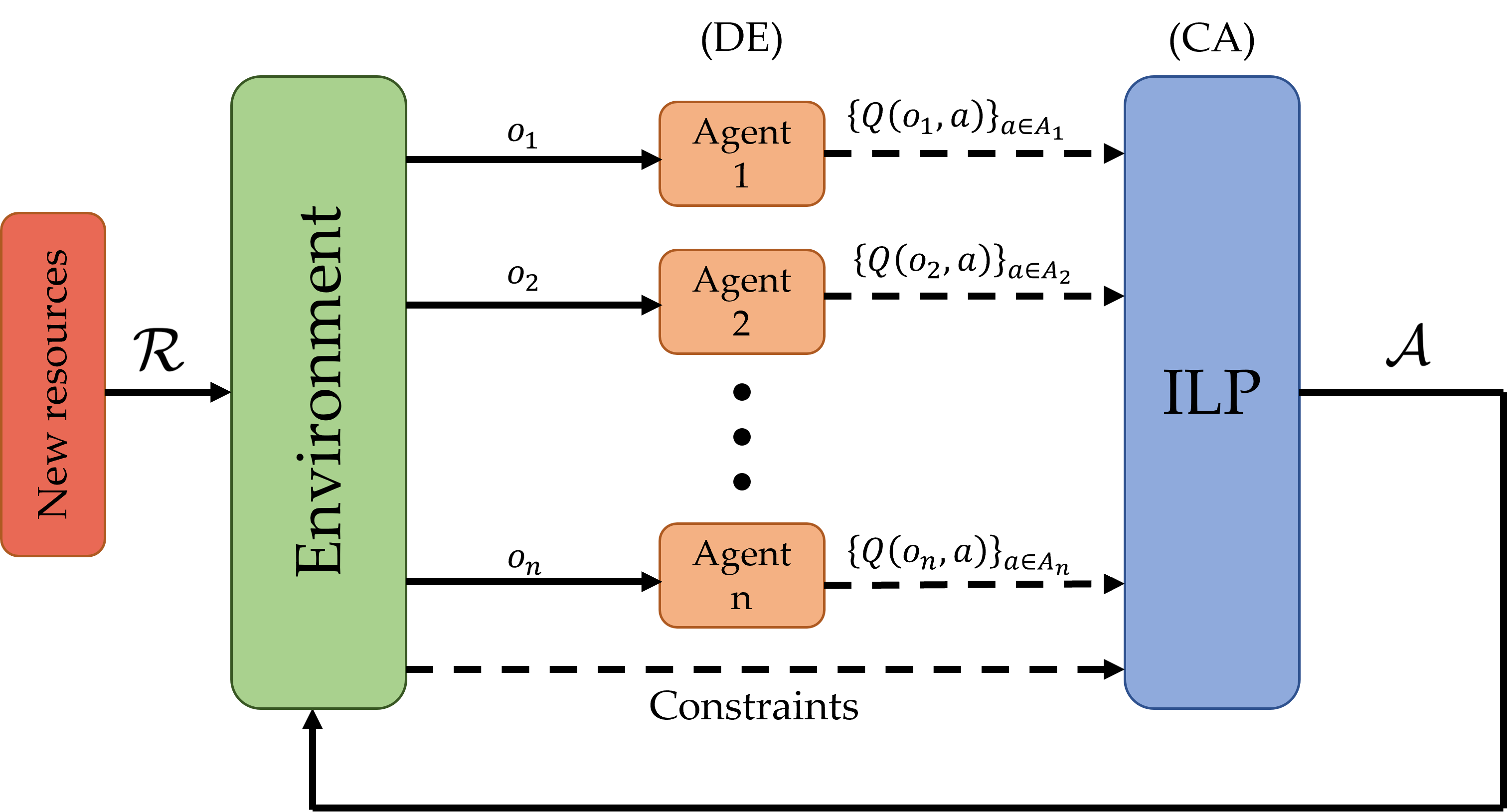}
    \smallskip
    \caption{An outline of the DECA pipeline. Each agent evaluates its available actions in a decentralized manner (DE), and the ILP finds the best joint action $\mathcal{A}$ using these evaluations and resource constraints (CA). }
    \label{fig:deca}
\end{figure}

\begin{figure*}[t]
    \centering
    \begin{subfigure}[t]{0.30\textwidth}
        \centering
        \includegraphics[height=1.5in]{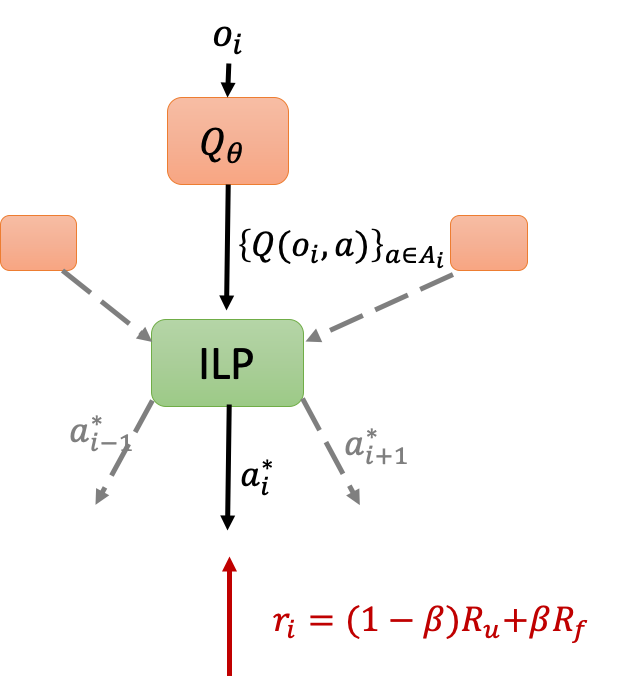}
        \caption{Joint Optimization (JO)}
    \end{subfigure}
    ~~
    \begin{subfigure}[t]{0.33\textwidth}
        \centering
        \includegraphics[height=1.5in]{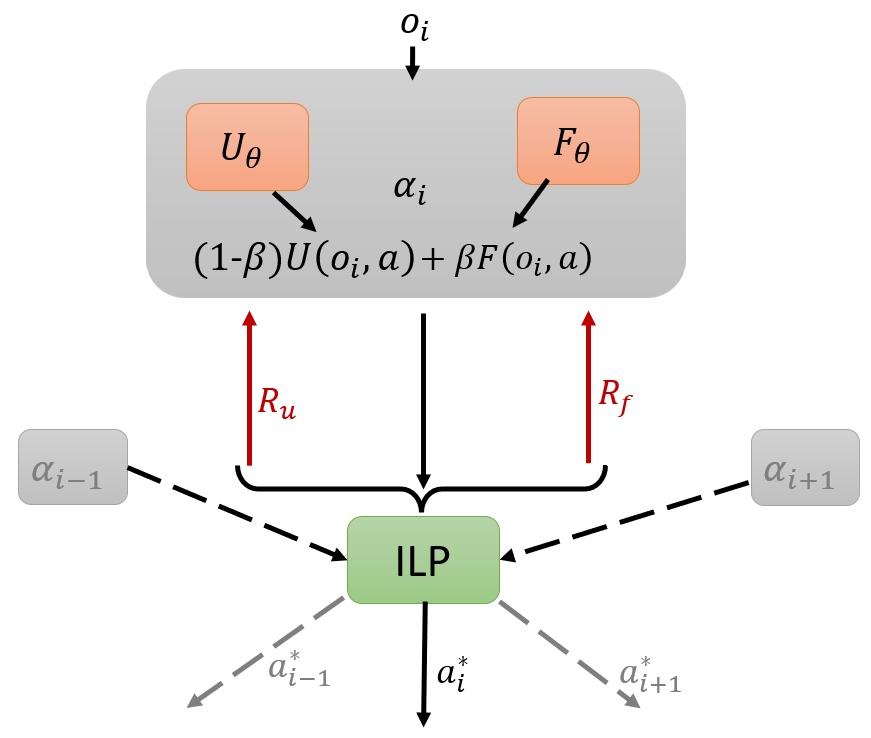}
        \caption{Split Optimization (SO)}
    \end{subfigure}
    ~~
    \begin{subfigure}[t]{0.33\textwidth}
        \centering
        \includegraphics[height=1.5in]{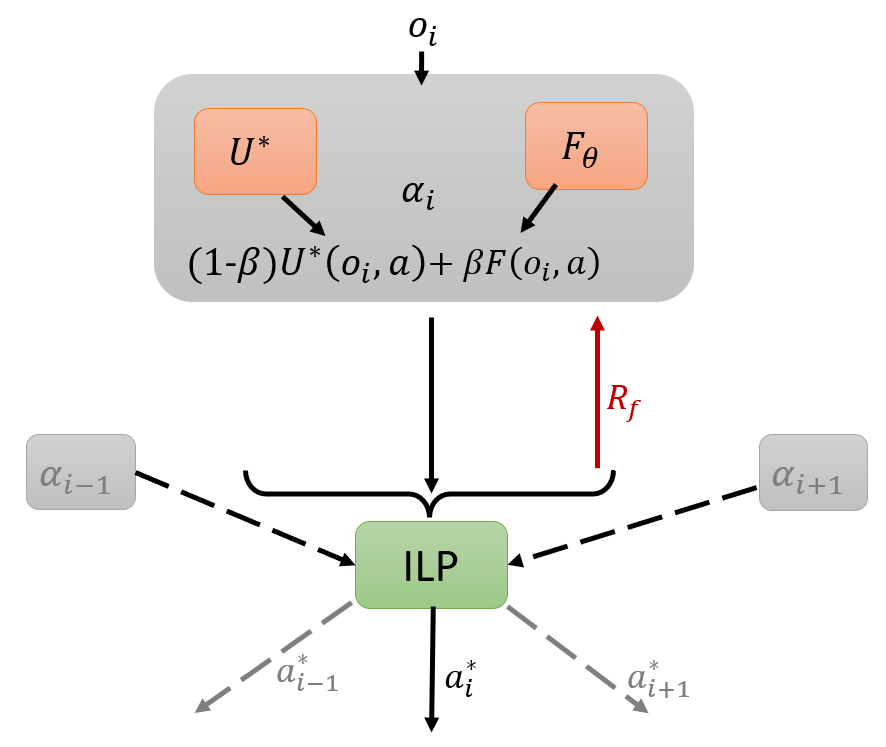}
        \caption{Fair-Only (FO)}
    \end{subfigure}%
    \caption{Illustration of our three DECAF methods to learn fairness. Each subfigure shows how the values propagate for a single agent. The red lines and text denote the actual reward to the learning model, which is used to update weights using TD learning. (a) Joint Optimization learns to predict a single combined value. (b) Split Optimization learns two separate estimators for utility and fairness, and combines their output. (c) Fair-Only assumes a black-box utility model $U^*$, and learns a fairness estimator only, combining their outputs to make decisions.}
    \label{fig:decaf}
\end{figure*}

\subsection{Optimization Framework}
The distributed evaluation (DE) step involves agents learning to predict the utilities of observation-action pairs using approaches such as Deep Q-learning~\citep{mnih2013playing,van2016DDQN}. The utility estimates are computed using partially-observable post-decision states, which are estimated locally, ignoring other agents' actions due to the infeasibility of exploring the joint state space.

The central allocation (CA) step solves a integer linear program that combines predicted utilities and resource constraints. 
Let $\mathcal{A}$ denote the allocation of actions decided by the central allocator such that $\mathcal{A}_i$ is the action assigned to agent $i$,
and let there be $K$ types of resources, each represented by $k\in\{1,2,\dots,K\}$, such that the number of resources available can be written as $\mathcal{R}\in\mathbb{R}^K$. 
We thus say that each resource type $k$ has an availability $\mathcal{R}_k$. This gives us the following optimization with decision variables $x_i(a)$:
\begin{align}
    \max_{x_i(a) \in \{0,1\}} \sum_{i \in \alpha} \sum_{a \in A_i} x_i(a) Q(o_i,a) \label{eq:Opt_DECAF}
\end{align}
subject to:
\begin{align}
    \sum_{a \in A_i, x_i(a) \in \{0,1\}} x_i(a) &= 1, \quad \forall i \in \alpha \label{eq:action_constraint}\\
    \sum_{a \in \mathcal{A}} c(a)_k \le \mathcal{R}_k, & \quad \forall k \in \{1, \ldots, K\} \label{eq:resource_constraint}
\end{align}
These constraints ensure that each agent is assigned exactly one action and that total resource usage does not exceed available supplies.
The ILP described above forms the central controller, and the Q-value estimator controls distributed evaluation. This offers benefits over completely distributed approaches by encapsulating resource constraints, and over completely centralized approaches by reducing the complexity of the learning objective. 
This setup is seen in many resource allocation problems~\citep{kube2019allocating, ride_alonso, shah2020neural}.

\section{DECAF: Fairness in DECAs}
\label{sec:fairness_in_decas}

In this section, we describe DECAF, our framework for incorporating fairness into DECA problems using Q-Learning. Specifically, we detail how we can specify and learn the fairness objective in Eq.~\ref{eq:objective} through the use of decomposed fairness rewards, and a modified Q-Learning algorithm to learn from these rewards using centralized training.

\subsection{Fairness Reward}

Previous work has considered learning to optimize a single social welfare function (SWF) that captures the notions of fairness and utility together, like the Coefficient of Variation used by FEN~\cite{jiang2019FEN} or $\alpha$-fairness and the Generalized Gini Function (GGF) used by SOTO~\cite{zimmer2021MOMDP}. With DECAF, we try to learn a class of objectives that trade off between system utility and fairness, characterized by a trade-off variable $\beta$ (Eq.~\ref{eq:objective}), with the aim of enabling flexible trade-offs between the two. 

Let $\textbf{Z}=\{z_i\}_{i\in \alpha}$ denote the vector of accumulated agent utilities (averaged or total). We interpret this utility as `accumulated wealth,' and look to make allocations that can result in a fairer distribution of this wealth. Let $\textbf{Z}_t^\pi$ represent the distribution of agent utility metrics at time $t$ following policy $\pi$. Then, we define a fairness function $\mathbb{F}:\mathbb{R}^{n}\rightarrow \mathbb{R}$ as a mapping of vector $\textbf{Z}$ to a real value, and our fairness objective is maximizing $\mathcal{F}_T = \mathbb{F}(\textbf{Z}_{t=T}^\pi)$.  Most popular SWFs and fairness functions can be cast into this form. 

In our work, we consider optimizing this objective by computing a fairness reward $R_f(\mathbf{s},\mathcal{A})$ that provides a vector of rewards, one for each agent, that captures how their current action contributed to the system fairness. We do this by using the per-step fairness change because of an allocation:
\begin{align}
    \Delta \mathcal{F}|\mathcal{A}^t 
    &= \mathcal{F}_{t+1} - \mathcal{F}_t \\
    &= \mathbb{F}(\textbf{Z}_{t+1}) - \mathbb{F}(\textbf{Z}_t) 
\end{align}
A naive way to decompose this reward is to evenly divide it among agents. This is commonly done in collaborative multi-agent RL when agents are optimizing a shared goal.
\begin{align}
    R_f(\mathbf{s}_t, \mathcal{A}^t) = \left[ \frac{\Delta \mathcal{F}|\mathcal{A}^t}{n} \right]_{i\in\alpha}
\end{align}
Alternatively, specialized decompositions can be designed to give a more informative signal to each agent. For example, if variance is used as the fairness function ($\mathcal{F}_t = -\var(\textbf{Z}_t)$):
\begin{align}
    \Delta \mathcal{F}|\mathcal{A}^t 
    &= -\var(\textbf{Z}_{t+1}) + \var(\textbf{Z}_t) \\
    &= -\frac{1}{n}\sum_{i\in\alpha}\left( z^{t+1}_i - \bar{z}_{t+1} \right)^2 + \frac{1}{n}\sum_{i\in\alpha}\left( z^{t}_i - \bar{z}_{t} \right)^2 \\
    R_f(\mathbf{s}, \mathcal{A}) &= \left[ -\frac{1}{n} \left( z_i^{t+1} - \bar{z}_{t+1} \right)^2 + \frac{1}{n}\left( z_i^{t} - \bar{z}_{t} \right)^2  \right]_{i \in \alpha} \label{eq:var_fair_reward}
\end{align}
Observe that this reward only depends on the agent's own metric value and the average metric. Thus, each iteration, apart from the local observation, each agent only needs to be communicated information about the average utility of all agents to reliably predict this value. This could be done by the central agent, or by message passing between the agents. 

For the main experiments of this paper, we use variance as our fairness function, with the reward function in Eq.~\ref{eq:var_fair_reward} as the fair reward. However, our methods are not limited to using variance. We also provide reward decompositions and experiments with other fairness functions including $\alpha$-fair, GGF, and maximin functions in the supplement, showing the generality of our approach. 

\subsection{Algorithms}

Given the fair reward $R_f$ described above, our approach targets the DE step to improve fairness, by changing the Q-values used in the ILP (Eq.~\ref{eq:Opt_DECAF}) to also account for fairness. We do this modifying $Q$ to be an estimator of the combined fair-efficient objective, with a weight $\beta\in [0,1]$ used to regulate relative value of fairness and utility. 

We use experience replay with centralized training to learn the Q-function, where an experience $\tau=\langle \mathbf{o},\mathcal{A},\mathbf{r}_u,\mathbf{r}_f,\mathbf{o}'\rangle$ stores a joint transition across all agents, with utility rewards $\mathbf{r}_u$ and fair rewards $\mathbf{r}_f$. Let $\theta$ denote the parameters of the Q-function. Given a replay buffer $\mathcal{D}$, we want to minimize the loss function $J_\theta=\mathbb{E}_{\tau\sim \mathcal{D}}L(\delta(\tau))$, where $\delta(\tau)$ is the Bellman error of the transition $\tau$, and $L$ is the MSE loss.
We propose three approaches for integrating fairness, illustrated in Figure~\ref{fig:decaf}. \footnote{Unless stated, we use bold terms to denote vectors, and overload Q-functions to also operate on vectors to compute a vector of outputs. Further, we use $Q(\mathbf{o})$ as a shorthand for computing Q-values for all possible actions for each observation in $\mathbf{o}$.}

\squishlist
    \item \textbf{Joint Optimization (JO):} A single estimator optimizes a weighted combination of fairness and utility.
    \begin{align}
    \delta(\tau) = (1-\beta) \mathbf{r}_u + \beta  \mathbf{r}_f + \gamma Q_\theta(\mathbf{o}') - Q_\theta(\mathbf{o}, \mathcal{A}) \label{eq:JO}
\end{align}
    \item \textbf{Split Optimization (SO):} Separate estimators for fairness ($F_\theta(\cdot)$) and utility ($U_\theta(\cdot)$) allow dynamic adjustment of their trade-off during policy execution.
    \begin{align}
    \delta^f(\tau) &= \mathbf{r}_f + \gamma F_\theta(\mathbf{o}') - F_\theta(\mathbf{o}, \mathcal{A}) \\
    \delta^u(\tau) &= \mathbf{r}_u + \gamma U_\theta(\mathbf{o}') - U_\theta(\mathbf{o}, \mathcal{A}) \\
    Q(\mathbf{o}, \mathcal{A}) &= (1-\beta) U_\theta(\mathbf{o}, \mathcal{A}) + \beta F_\theta(\mathbf{o}, \mathcal{A}) \label{eq:SO}
\end{align}
    \item \textbf{Fair-Only Optimization (FO):} A fairness estimator ($F_\theta(\cdot)$) adjusts a pre-existing utility function $U^*(\cdot)$ to incorporate fairness, useful when utility functions are provided externally.
    \begin{align}
    \delta^f(\tau) &= \mathbf{r}_f(s,a) + \gamma F_\theta(\mathbf{o}') - F_\theta(\mathbf{o}, \mathcal{A}) \\
    Q(\mathbf{o}, A) &= (1-\beta) U^*(\mathbf{o}, \mathcal{A}) + \beta F_\theta(\mathbf{o}, \mathcal{A}) \label{eq:FO}
\end{align}
\squishend

Our learning algorithm is based on Double Deep Q-Learning~\cite{van2016DDQN}, which uses a target network to stabilize updates. The key differentiating factor is in how the target values are computed. When learning from an experience, we compute the optimal action $\mathcal{A}^*$ in the successor state by solving the ILP~(Eq.~\ref{eq:Opt_DECAF}) using the online Q-network, and then compute the Q-value of the selected actions using the target network. The models are updated using the rewards (stored in the experience) obtained after the ILP allocation of the previous state, as shown in the red text and arrows in Figure~\ref{fig:decaf}. 
For SO and FO, we package the utility and fairness estimators into the Q-function, and compute the optimal successor action using both. Then, we independently update each estimator using the TD error of their respective objectives. We provide the pseudocode for the learning algorithms in Appendix \ref{sec:appendixAlgo}.
For FO, we skip training the utility estimator.
SO and FO offer the additional benefits of interpretability, as during execution, we are able to discern how much of the decision was based on the utility gain and fairness improvement respectively.

SO also provides some useful properties described below.
\begin{theorem}
\label{th:theorem_fair}
Given perfect estimates for utility and fairness, increasing $\beta$ always improves the one-step fairness gain for SO with $\gamma=0$.
\end{theorem}
\noindent \textbf{Proof Sketch:} We show this using the property of the ILP. The only way another action is selected when $\beta$ is increased is if the objective value of that action is higher. Since $\beta\ge0$, comparing the objective value of two allocations shows us that for $\beta'>\beta$, this only happens when the fairness gain for the new action is higher. $\hfill \Box$

\begin{theorem}
\label{th:theorem_fairest}
For a large enough $\beta$, the fairest allocation will be selected with perfect utility and fairness estimators for SO with $\gamma=0$.
\end{theorem}
\noindent \textbf{Proof Sketch:}
Since the fairest allocation has the largest fairness gain, we show that there exists a $\beta_f$ such that the fairness gain's contribution to the objective outweighs any utility loss, making it the optimal allocation. $\hfill \Box$

\medskip
We also provide the corollaries to these theorems for improving utility as $\beta$ is reduced in Appendix \ref{sec:appendixTheoretical}, along with full proofs for these theorems. 
These properties also empirically hold when $\gamma\ne0$, as our experiments demonstrate. This adaptability is a major strength of SO: It allows a degree of flexibility that other methods do not possess. Specifically, SO allows users to vary the trade-off weight $\beta$ during runtime, and the behavior can be expected to be monotonic in the direction of change. For $\gamma=0$, Theorem~\ref{th:theorem_fair} and its corollary guarantee that the space of selected allocations is Pareto-efficient with changing $\beta$ at each time step.

\begin{figure*}[ht]
    \centering
    \includegraphics[width=0.98\linewidth]{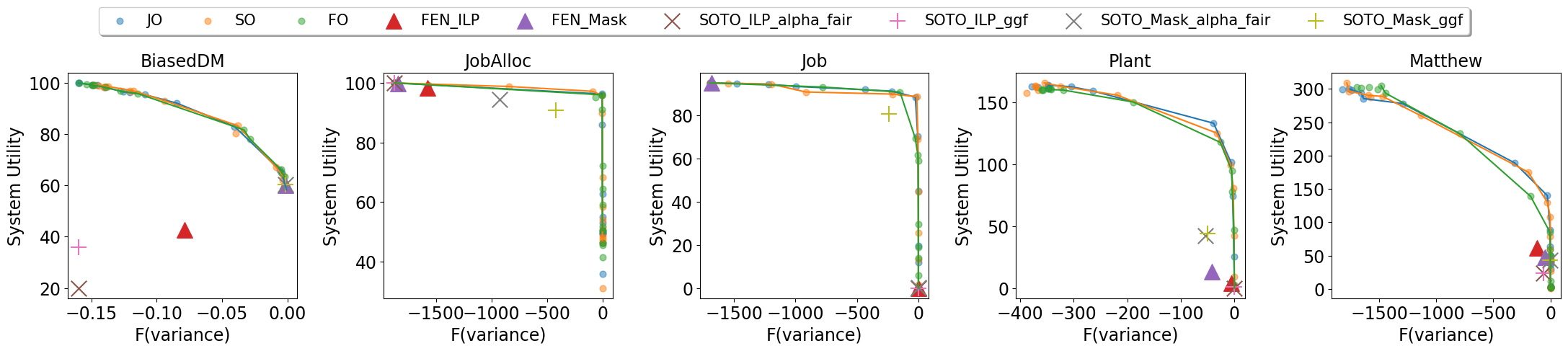}
    \caption{Change in system utility and fairness as $\beta$ is increased, with $\beta=0$ at the top left $\beta=1$ at the bottom-right. For all domains, we can see that split and joint optimization perform similarly, while learning only fairness can sometimes be slightly worse. All our methods Pareto-dominate SOTO and FEN. Each point depicts the average performance over five different models trained at that $\beta$ value, and the lines show the Pareto front for each method.}
    \label{fig:main_results}
\end{figure*}

\section{Experimental Setup}

We conduct experiments for maximizing the objective in Eq.~\ref{eq:objective}, where the system utility is the sum of all agent utilities at the end of an episode, and the fairness is measured as the negative of the variance of agent resources at the end of the episode. We perform experiments for a variety of $\beta$ values, repeating each configuration 5 times for each of our three settings: \textbf{Joint Optimization (JO)}, \textbf{Split Optimization (SO)} and \textbf{Fair-Only Optimization (FO)}. 
We were unable to use off-the-shelf multi-agent RL libraries because of their lack of support for constrained central decision making. Thus, we also implemented our own versions of the learning algorithm (DDQN with $\epsilon$-greedy TD(0) learning), as described in Section~\ref{sec:fairness_in_decas}.
Each model uses the same network architecture, with two hidden layers of dimension 20, and  output of dimension 1.
The utility model used for FO is randomly selected from the JO models trained with $\beta=0$. We included features indicating the relative advantage of each agent as a signal for fairness, in addition to the features describing the local observation of each agent.

\subsection{Environments}
We adapt the environments from \citet{jiang2019FEN} to align with the DECA framework, reformulating them as resource allocation problems with constraints. Additionally, we introduce a new environment, BiasedDM, featuring a biased decision-maker with differing utilities for agents. Below is a description of each environment.

\noindent \textbf{Matthew:} This environment showcases the Matthew effect~\citep{rigney2010matthew, gao2023matthew}, where the rich get richer. Ten agents move on a continuous unit grid with three resources available at a time. Consuming resources grants agents speed and size boosts, allowing faster access to future resources. Some agents start with inherent advantages. Actions involve assigning resources to agents, preventing others from accessing them, or taking a null action to move randomly. Agents provide utility estimates for each action, and the decision-maker allocates resources, ensuring no two agents share the same resource. Agents always move in straight lines toward their targets and are unable to target new resources while moving to collect a previously allocated resource.

\noindent \textbf{Job:} Four agents operate on a discretized grid with a fixed square containing a job. Agents receive rewards for occupying the job's location. Grid locations act as resources, with only one agent allowed per location. Agents move in cardinal directions and communicate directional preferences to the decision-maker, who assigns final moves.

\noindent \textbf{JobAlloc:} This simplified version of the Job environment removes the grid. Agents directly compete to occupy the job, with actions limited to occupying or leaving it. The job can only be claimed if it is unoccupied. This domain's challenge is overcoming the single-step suboptimality when no agent occupies the job.

\noindent \textbf{Plant:} Five agents operate on a discretized grid containing eight resources of three types. Agents must collect specific resource combinations to construct a `unit' and earn rewards. Requirements vary in difficulty across agents. The decision-maker assigns resources based on agents' preferences, ensuring exclusivity. Agents deterministically move toward assigned resources.

\noindent \textbf{BiasedDM:} Unlike the other environments, this environment introduces an explicit bias in decision-making. Five agents compete for a single resource per timestep, where utility to the decision-maker increases with agent index ($0.2 \times i$ for agent $i$). Optimal utility is achieved by always allocating resources to agent 5. Fairness is assessed based on resource distribution over time, highlighting a disconnect between fairness and utility.

\subsection{Baselines}

As we described in our Related Work section, FEN~\cite{jiang2019FEN} and SOTO~\cite{zimmer2021MOMDP} are two most relevant approaches that are generalizable to different domains. We thus compare against them in our experimental evaluations. 
However, they both operate in environments where agents can independently take actions without explicitly accounting for resource constraints, making them incompatible with the DECA framework. We attempt two methods of making constrained decisions with per-agent policies to adapt them:
\squishlist
    \item \textbf{Policy as Q-values:} We treat the action probabilities as Q-values and use them for the central allocation. This is denoted by the \textbf{`\_ILP'} suffix in the experiments.
    \item \textbf{Masked sequential action selection:} We go through agents sequentially and let them sample an action from their policy and assign it to them. Any invalid actions are masked as the resources get consumed. We randomize the order of agents every step to prevent ordering bias. This is denoted by the \textbf{`\_Mask'} suffix in the experiments.
\squishend
We add the extra features that SOTO requires (only for SOTO), and train SOTO with both the $\alpha$-fair and $GGF$ objective described in their paper~\cite{zimmer2021MOMDP}. We also use shared weights across agents in our experiments.

\begin{table*}[t]
\centering
\caption{Evaluation of all models on multiple metrics for each environment. For JO, SO, and FO, the values in the bracket denote the $\beta$ value selected based on the model that maximizes 0.1$\cdot U$ - 0.9$\cdot \var(\textbf{Z})$. The selected $\beta$ value is indicated in brackets. The values in bold are the best in each row.}
\label{tab:BaselineResults}
\resizebox{\textwidth}{!}{
\begin{tabular}{clcccccc}
\hline
\multicolumn{1}{l}{\textbf{Environment}} & \textbf{Metric} & \textbf{JO($\beta$)} & \textbf{SO($\beta$)}  & \textbf{FO($\beta$)} & \textbf{FEN}   & \textbf{SOTO($\alpha$-Fair)} & \textbf{SOTO(GGF)} \\
\hline
\multirow{5}{*}{BiasedDM}                & Alpha Fair      & \textbf{-8.09(1)}    & -8.3(0.999)           & -8.19(0.9995)        & -8.15          & -8.15                        & -8.15              \\
                                         & GGF             & \textbf{0.35(1)}     & 0.3(0.999)            & 0.33(0.9995)         & 0.33           & 0.33                         & 0.33               \\
                                         & Maximin         & \textbf{0.16(1)}     & 0.12(0.999)           & 0.15(0.9995)         & 0.15           & 0.15                         & 0.15               \\
                                         & System Utility  & 58.31(1)             & \textbf{63.65(0.999)} & 63.38(0.9995)        & 59.95          & 60.22                        & 60.24              \\
                                         & Variance        & \textbf{-0.0007(1)}  & -0.0033(0.999)        & -0.0022(0.9995)      & -0.0014        & -0.0015                      & -0.0015            \\
\hline
\multirow{5}{*}{JobAlloc}                & Alpha Fair      & \textbf{12.71(0.2)}  & 12.69(0.2)            & \textbf{12.71(0.2)}  & -35.31         & -7.39                        & 1.06               \\
                                         & GGF             & 43.11(0.2)           & 43.41(0.2)            & \textbf{43.8(0.2)}   & 12.59          & 19.5                         & 29.62              \\
                                         & Maximin         & 21.79(0.2)           & 22.21(0.2)            & \textbf{22.71(0.2)}  & 0              & 3.49                         & 10.62              \\
                                         & System Utility  & 96.28(0.2)           & 95.81(0.2)            & 96(0.2)              & \textbf{99.89} & 94.42                        & 90.92              \\
                                         & Variance        & -4.44(0.2)           & -2.54(0.2)            & -1.48(0.2)           & -1839.76       & -923.46                      & -421.93            \\
\hline
\multirow{5}{*}{Job}                & Alpha Fair      & 10.97(0.2)           & \textbf{12.07(0.2)}   & 10.77(0.2)           & -35.89         & -55.21                       & 5.03               \\
                                         & GGF             & 37.12(0.2)           & \textbf{37.88(0.2)}   & 26.65(0.2)           & 11.87          & 0                            & 22.17              \\
                                         & Maximin         & 16.99(0.2)           & \textbf{18.13(0.2)}   & 13.13(0.2)           & 0              & 0                            & 4.9                \\
                                         & System Utility  & 88.43(0.2)           & 88.63(0.2)            & 61.68(0.2)           & \textbf{94.88} & 0                            & 80.57              \\
                                         & Variance        & -25.14(0.2)          & -13.38(0.2)           & -4.59(0.2)           & -1683.49       & \textbf{0}                   & -242.3             \\
\hline
\multirow{5}{*}{Plant}                   & Alpha Fair      & \textbf{15(0.8)}     & 14.77(0.8)            & 14.54(0.8)           & -35.48         & -20.75                       & -20.62             \\
                                         & GGF             & \textbf{35.67(0.8)}  & 34.63(0.8)            & 33.45(0.8)           & 1.4            & 10.28                        & 10.89              \\
                                         & Maximin         & \textbf{16.61(0.8)}  & 15.98(0.8)            & 15.74(0.8)           & 0              & 3.62                         & 4.08               \\
                                         & System Utility  & \textbf{101.72(0.8)} & 99.38(0.8)            & 94.89(0.8)           & 13.31          & 42.57                        & 43.84              \\
                                         & Variance        & -6.34(0.8)           & -6.75(0.8)            & \textbf{-4.99(0.8)}  & -41.64         & -54.29                       & -49.91             \\
\hline
\multirow{5}{*}{Matthew}                 & Alpha Fair      & 20.03(0.2)           & \textbf{23.4(0.5)}    & 20.71(0.5)           & -29.04         & 12.45                        & 12.64              \\
                                         & GGF             & 14.41(0.2)           & \textbf{19(0.5)}      & 11.92(0.5)           & 1.73           & 4.6                          & 4.67               \\
                                         & Maximin         & 3.64(0.2)            & \textbf{8.86(0.5)}    & 5.09(0.5)            & 0.36           & 1.66                         & 1.69               \\
                                         & System Utility  & \textbf{140(0.2)}    & 108.1(0.5)            & 85.5(0.5)            & 47.39          & 42.77                        & 43.02              \\
                                         & Variance        & -26.95(0.2)          & \textbf{-1.28(0.5)}   & -5.17(0.5)           & -45.9          & -3.33                        & -3.32      \\       
\hline
\end{tabular}
}
\end{table*}
\section{Results}

Figure \ref{fig:main_results} shows the performance of all three DECAF methods (JO, SO, FO) on the five domains discussed above. For each method, we varied the hyperparameter $\beta$ controlling the fairness-utility trade-off (Eqs.~\ref{eq:JO}, \ref{eq:SO}, \ref{eq:FO}), starting with $\beta=0$ (top-left) and increasing to $\beta=1$ (bottom-right).  As mentioned in Section \ref{sec:fairness_in_decas}, we present results where we optimize for variance as the fairness function here. Additional results on learning with different fairness functions are included in the supplement.

\subsection{Efficacy of the Fairness-Utility Optimization}
For all domains, all three methods are able to learn expressive policies which lie at various points close to the Pareto front. This shows that the optimization allows the model to trade off utility and fairness to show diverse behaviors as required by the user. This also confirms that the fairness reward proposed for minimizing variance is a good signal.

\subsection{Comparison Against Baselines}

As seen in Figure~\ref{fig:main_results}, our methods Pareto-dominate both FEN and SOTO in all experiments, with SOTO\_Mask with GGF being the most competitive. 
Since the baselines were not trained on variance, we also compare the performance of DECAF on other metrics of interest.
For a more granular comparison, we select one $\beta$ value for JO, SO, and FO each, and compare it to the other methods across a variety of metrics. 
Table~\ref{tab:BaselineResults} shows the results for all domains, where we see that our methods provide better results all across the board.
Between the masked and ILP versions of the baselines, the masked versions perform better. This is to be expected, as using the ILP to select the best actions results in trajectories that are not on-policy for each agent, which breaks the requirements for policy gradient methods. The masked approach, on the other hand, cannot benefit from the centralized decision-making, and the random order of agents can lead to suboptimal behavior and missed opportunities. Further, it is difficult to extend policy gradient methods to variable or combinatorial action spaces, while our approach allows for arbitrary action spaces, as long as reasonable post-decision states can be approximated. Thus, in the DECA setting, Q-learning based methods like DECAF have the upper edge.

\begin{figure}[t]
    \centering
    \subfloat[System Utility]{\includegraphics[width=0.475\linewidth]{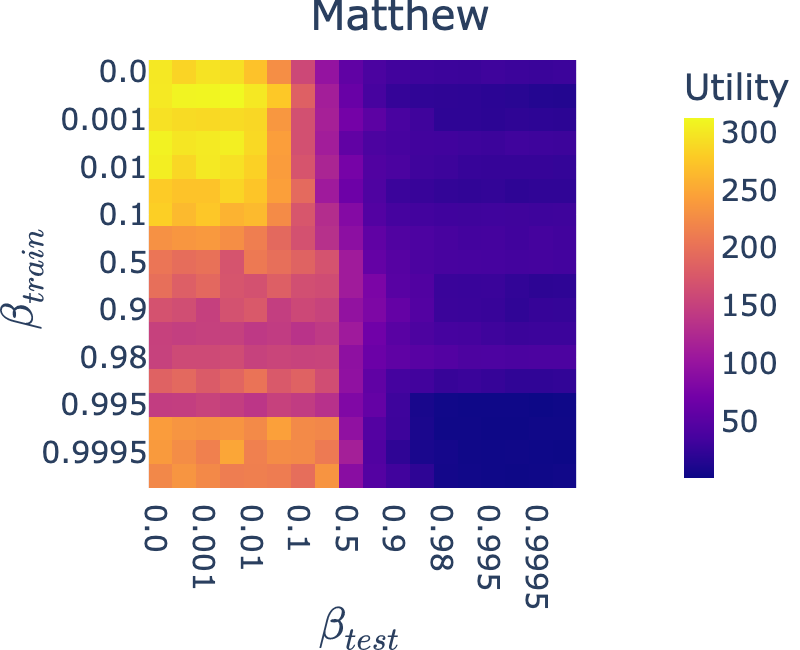}
    \label{fig:matthew-a}}
    \subfloat[Variance]{\includegraphics[width=0.475\linewidth]{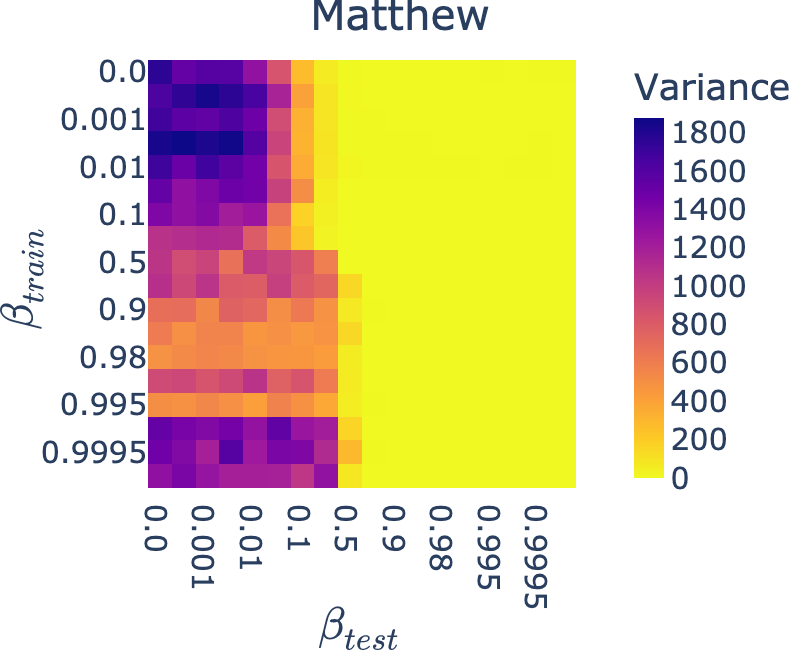}
    \label{fig:matthew-b}}
    \caption{Evaluation of SO models trained on $\beta_{train}$ and evaluated on $\beta_{test}$ for the Matthew environment. Brighter colors indicate better outcomes.}
    \label{fig:matthew}
\end{figure}

\begin{figure}[t]
    \centering
    \subfloat[System Utility]{    
    \includegraphics[width=0.475\linewidth]{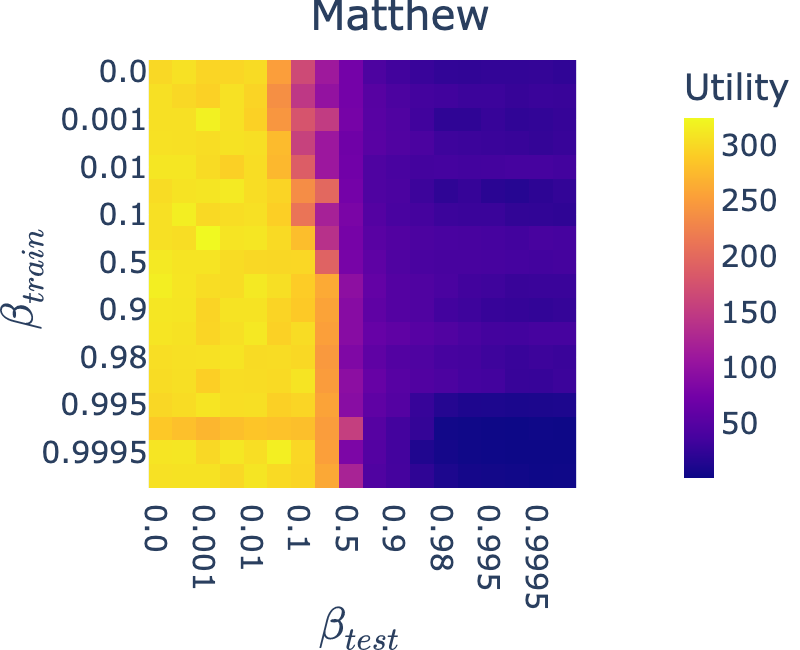}
    \label{fig:matthew_FO-a}}
    \subfloat[Variance]{    \includegraphics[width=0.475\linewidth]{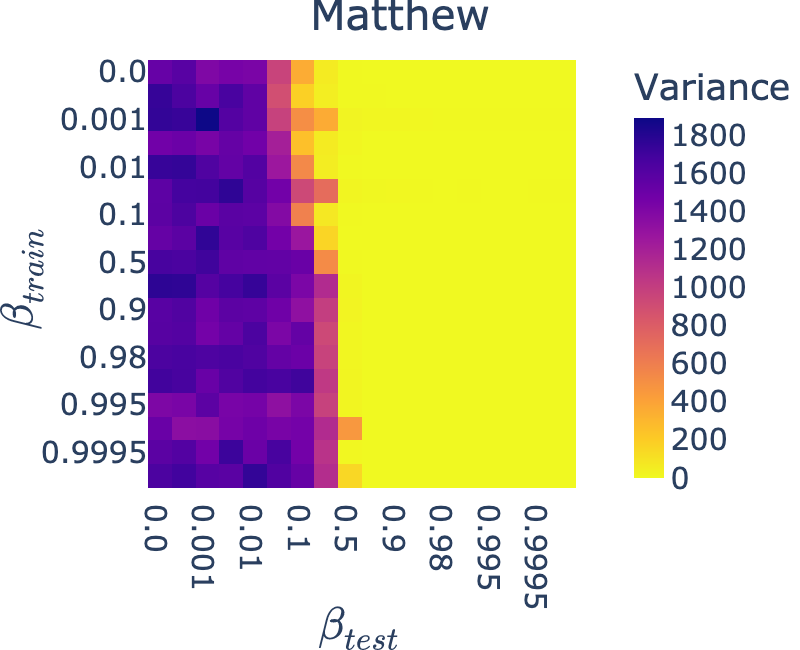}
    \label{fig:matthew_FO-b}}
    \caption{Evaluation of FO models trained on $\beta_{train}$ and evaluated on $\beta_{test}$ for the Matthew environment. Brighter colors indicate better outcomes.}
    \label{fig:matthew_FO}
\end{figure}
\begin{figure}[t]
    \centering
    \subfloat[Generalization of SO]{
    \includegraphics[width=0.45\linewidth]{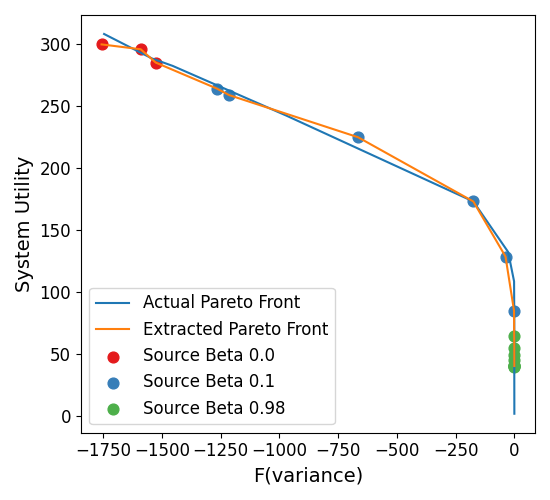}
    \label{fig:matthew-apprx-Pareto-SO}}
    \subfloat[Generalization of FO]{    \includegraphics[width=0.45\linewidth]{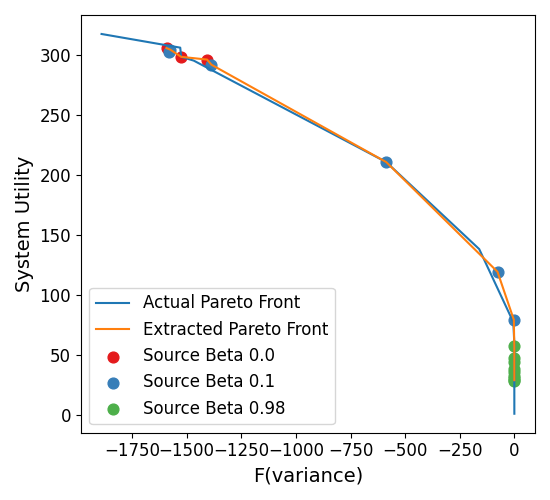}
    \label{fig:matthew-apprx-Pareto-FO}}
    \caption{Approximated Pareto fronts using sparse $\beta_{train}$ evaluated on other $\beta$ values for the Matthew domain.}
    \label{fig:matthew_apprx_Pareto}
\end{figure}

\subsection{Comparison of DECAF Methods}
In our results, JO and SO generally exhibit similar performance characteristics, suggesting that simultaneous evolution of utility and fairness estimates is beneficial. FO is also very competetive, but in some cases, it falls below the Pareto front.
This underperformance is likely due to out-of-distribution transitions for the fixed utility model, which are more problematic in FO when a large fairness weight $\beta$ is used, causing a larger shift in the state distribution and resulting in degraded utility estimates. We expect this to be a bigger issue in more complicated environments, or with poor black-box utility functions.

\subsection{Generalization Using Split Optimization (SO)}
Figure~\ref{fig:matthew} shows detailed results for the Matthew environment, when using SO. We evaluate each model trained on a particular $\beta_{train}$ on all other $\beta_{test}$. This allows us to see how well the trained fairness and utility models generalize when the operating $\beta$ is changed. Note that for all these models, $\beta$ is not provided as a feature to the Q-function. 

The diagonal elements show the behavior when training and testing is done on the same $\beta$ value.  
From the plots for system utility (Figure \ref{fig:matthew-a}) and variance (Figure \ref{fig:matthew-b}), we can see that as $\beta_{test}$ increases, variance improves, and as $\beta_{test}$ decreases, utility improves. This is the expected behavior, and the major advantage of SO over JO. With JO, the model only predicts a single value, so we are unable to change the trade-off weight during evaluation, and we require a unique model for each $\beta$ that we want the model to work for. However, with SO, selecting just a few spread out $\beta$ values can allow us to extrapolate between them, providing online adaptability.
This shows that SO has the flexibility to function well at operating points away from the $\beta_{train}$ that it is trained for.

Figure~\ref{fig:matthew-apprx-Pareto-SO} shows how well a few selected models can generalize to the Pareto front. We pick $\beta_{train}$ values evenly spaced across the search space, and evaluate the model on all $\beta_{test}$, picking the closest $\beta_{train}$ in order of the search space. We can see that the approximated Pareto front closely matches the actual Pareto front, even with just 3 models, further demonstrating the strength of SO.
These observations hold for other domains as well, and the results are included in the Supplement.

\subsection{Effectiveness of Fair-Only Optimization (FO)}
Like SO, FO is also able to generalize well when different $\beta_{test}$ are used to evaluate the learned models (Figure~\ref{fig:matthew_FO}). Because the utility model is fixed, all models achieve high utility as $\beta_{test}\xrightarrow{}0$ (Figure~\ref{fig:matthew_FO-a}). Further, all models also improve fairness as $\beta_{test}$ grows larger (Figure~\ref{fig:matthew_FO-b}). The behavior change from utility-oriented to fairness-oriented is much sharper in FO when compared to SO. 
Looking at Figure~\ref{fig:matthew-apprx-Pareto-FO}, we can again see that even FO has the ability to generalize from only a few models to cover the entire Pareto front. 
Despite being Pareto-dominated by SO and JO at intermediate $\beta$ values in some domains, FO has the advantage of reliability: A trusted black-box utility model can be used in conjunction with a possibly smaller fairness model, with the guarantee to behave optimally as $\beta_{test}$ is reduced. When such a model is available, FO is the best choice, given its competent performance and lower computational load.

\section{Conclusions and Future Work}
We proposed DECAF, a framework for learning long-term utility and fairness estimates in multi-agent resource allocation. DECAF is among the first approaches to optimize fair resource allocation under resource constraints, supporting diverse problem settings by decoupling fairness and utility metrics. Split and Fair-Only optimization enable online trade-offs between utility and fairness without retraining, enhancing interpretability. Our results demonstrate the flexibility and effectiveness of our approaches across various scenarios.
 
Our framework currently relies on Q-Learning, as deriving a policy gradient approach for DECA problems is challenging due to the dynamic state-action space and the indirect relationship between agent `policies' and actions resulting from ILP optimization. Addressing this challenge is a promising direction for future research.
Finally, our methods are not the only way to decompose the fairness reward across agents. Techniques like VDN~\citep{sunehag2017VDN} or QMIX~\citep{rashid2020QMIX} could be integrated with our framework to learn credit assignment for fair rewards.

\bibliography{main}

\begin{thebibliography}{23}
\providecommand{\natexlab}[1]{#1}
\providecommand{\url}[1]{\texttt{#1}}
\expandafter\ifx\csname urlstyle\endcsname\relax
  \providecommand{\doi}[1]{doi: #1}\else
  \providecommand{\doi}{doi: \begingroup \urlstyle{rm}\Url}\fi

\bibitem[Alonso-Mora et~al.(2017)Alonso-Mora, Samaranayake, Wallar, Frazzoli, and Rus]{ride_alonso}
Alonso-Mora, J., Samaranayake, S., Wallar, A., Frazzoli, E., and Rus, D.
\newblock On-demand high-capacity ride-sharing via dynamic trip-vehicle assignment.
\newblock \emph{Proceedings of the National Academy of Sciences}, 114:\penalty0 462--467, 2017.

\bibitem[De~Nijs et~al.(2021)De~Nijs, Walraven, De~Weerdt, and Spaan]{CMMDP_de2021}
De~Nijs, F., Walraven, E., De~Weerdt, M., and Spaan, M.
\newblock Constrained multiagent {Markov} decision processes: A taxonomy of problems and algorithms.
\newblock \emph{Journal of Artificial Intelligence Research}, 70:\penalty0 955--1001, 2021.

\bibitem[Dwork et~al.(2012)Dwork, Hardt, Pitassi, Reingold, and Zemel]{DP_dwork2012}
Dwork, C., Hardt, M., Pitassi, T., Reingold, O., and Zemel, R.
\newblock Fairness through awareness.
\newblock In \emph{Proceedings of the Conference on Innovations in Theoretical Computer Science}, pp.\  214--226, 2012.

\bibitem[Elmachtoub \& Grigas(2022)Elmachtoub and Grigas]{elmachtoub2022smartPtO}
Elmachtoub, A.~N. and Grigas, P.
\newblock Smart “predict, then optimize”.
\newblock \emph{Management Science}, 68:\penalty0 9--26, 2022.

\bibitem[Gajane et~al.(2022)Gajane, Saxena, Tavakol, Fletcher, and Pechenizkiy]{FairRLSurvey}
Gajane, P., Saxena, A., Tavakol, M., Fletcher, G., and Pechenizkiy, M.
\newblock Survey on fair reinforcement learning: Theory and practice.
\newblock \emph{arXiv preprint arXiv:2205.10032}, 2022.

\bibitem[Gao et~al.(2023)Gao, Huang, Chen, Zhang, Li, Jiang, Wang, Zhang, and He]{gao2023matthew}
Gao, C., Huang, K., Chen, J., Zhang, Y., Li, B., Jiang, P., Wang, S., Zhang, Z., and He, X.
\newblock Alleviating matthew effect of offline reinforcement learning in interactive recommendation.
\newblock In \emph{Proceedings of the International Conference on Research and Development in Information Retrieval}, pp.\  238--248, 2023.

\bibitem[Hardt et~al.(2016)Hardt, Price, and Srebro]{hardt2016equality}
Hardt, M., Price, E., and Srebro, N.
\newblock Equality of opportunity in supervised learning.
\newblock In \emph{Proceedings of the Conference on Neural Information Processing Systems}, pp.\  3323--3331, 2016.

\bibitem[Hasselt et~al.(2016)Hasselt, Guez, and Silver]{van2016DDQN}
Hasselt, H.~v., Guez, A., and Silver, D.
\newblock Deep reinforcement learning with double {Q}-learning.
\newblock In \emph{Proceedings of the AAAI Conference on Artificial Intelligence}, pp.\  2094--2100, 2016.

\bibitem[Jiang \& Lu(2019)Jiang and Lu]{jiang2019FEN}
Jiang, J. and Lu, Z.
\newblock Learning fairness in multi-agent systems.
\newblock In \emph{Proceedings of the Conference on Neural Information Processing Systems}, 2019.

\bibitem[Kube et~al.(2019)Kube, Das, and Fowler]{kube2019allocating}
Kube, A.~R., Das, S., and Fowler, P.~J.
\newblock Allocating interventions based on predicted outcomes: A case study on homelessness services.
\newblock In \emph{Proceedings of the AAAI Conference on Artificial Intelligence}, pp.\  622--629, 2019.

\bibitem[Kube et~al.(2023)Kube, Das, and Fowler]{kube2023community}
Kube, A.~R., Das, S., and Fowler, P.~J.
\newblock Community-and data-driven homelessness prevention and service delivery: optimizing for equity.
\newblock \emph{Journal of the American Medical Informatics Association}, 30\penalty0 (6):\penalty0 1032--1041, 2023.

\bibitem[Kumar et~al.(2023)Kumar, Vorobeychik, and Yeoh]{SI_kumar2023}
Kumar, A., Vorobeychik, Y., and Yeoh, W.
\newblock Using simple incentives to improve two-sided fairness in ridesharing systems.
\newblock In \emph{Proceedings of the International Conference on Automated Planning and Scheduling}, pp.\  227--235, 2023.

\bibitem[Mehrabi et~al.(2021)Mehrabi, Morstatter, Saxena, Lerman, and Galstyan]{mehrabi2021MLBiassurvey}
Mehrabi, N., Morstatter, F., Saxena, N., Lerman, K., and Galstyan, A.
\newblock A survey on bias and fairness in machine learning.
\newblock \emph{ACM Computing Surveys}, 54\penalty0 (6):\penalty0 1--35, 2021.

\bibitem[Mnih et~al.(2013)Mnih, Kavukcuoglu, Silver, Graves, Antonoglou, Wierstra, and Riedmiller]{mnih2013playing}
Mnih, V., Kavukcuoglu, K., Silver, D., Graves, A., Antonoglou, I., Wierstra, D., and Riedmiller, M.
\newblock Playing atari with deep reinforcement learning.
\newblock \emph{arXiv preprint arXiv:1312.5602}, 2013.

\bibitem[Qin et~al.(2022)Qin, Zhu, and Ye]{qin2022RLforRides}
Qin, Z.~T., Zhu, H., and Ye, J.
\newblock Reinforcement learning for ridesharing: An extended survey.
\newblock \emph{Transportation Research Part C: Emerging Technologies}, 144:\penalty0 103852, 2022.

\bibitem[Raghavan et~al.(2020)Raghavan, Barocas, Kleinberg, and Levy]{raghavan2020hiringbias}
Raghavan, M., Barocas, S., Kleinberg, J., and Levy, K.
\newblock Mitigating bias in algorithmic hiring: Evaluating claims and practices.
\newblock In \emph{Proceedings of the Conference on Fairness, Accountability, and Transparency}, pp.\  469--481, 2020.

\bibitem[Rashid et~al.(2020)Rashid, Samvelyan, De~Witt, Farquhar, Foerster, and Whiteson]{rashid2020QMIX}
Rashid, T., Samvelyan, M., De~Witt, C.~S., Farquhar, G., Foerster, J., and Whiteson, S.
\newblock Monotonic value function factorisation for deep multi-agent reinforcement learning.
\newblock \emph{Journal of Machine Learning Research}, 21\penalty0 (178):\penalty0 1--51, 2020.

\bibitem[Rigney(2010)]{rigney2010matthew}
Rigney, D.
\newblock \emph{The Matthew Effect: How Advantage Begets Further Advantage}.
\newblock Columbia University Press, 2010.

\bibitem[Shah et~al.(2020)Shah, Lowalekar, and Varakantham]{shah2020neural}
Shah, S., Lowalekar, M., and Varakantham, P.
\newblock Neural approximate dynamic programming for on-demand ride-pooling.
\newblock In \emph{Proceedings of the AAAI Conference on Artificial Intelligence}, pp.\  507--515, 2020.

\bibitem[Siddique et~al.(2020)Siddique, Weng, and Zimmer]{siddique2020MOMDP}
Siddique, U., Weng, P., and Zimmer, M.
\newblock Learning fair policies in multi-objective (deep) reinforcement learning with average and discounted rewards.
\newblock In \emph{Proceedings of the International Conference on Machine Learning}, pp.\  8905--8915, 2020.

\bibitem[Sunehag et~al.(2018)Sunehag, Lever, Gruslys, Czarnecki, Zambaldi, Jaderberg, Lanctot, Sonnerat, Leibo, Tuyls, and Graepel]{sunehag2017VDN}
Sunehag, P., Lever, G., Gruslys, A., Czarnecki, W.~M., Zambaldi, V., Jaderberg, M., Lanctot, M., Sonnerat, N., Leibo, J.~Z., Tuyls, K., and Graepel, T.
\newblock Value-decomposition networks for cooperative multi-agent learning based on team reward.
\newblock In \emph{Proceedings of the Conference on Autonomous Agents and Multiagent Systems}, pp.\  2085--2087, 2018.

\bibitem[Wang et~al.(2021)Wang, Shah, Chen, Perrault, Doshi-Velez, and Tambe]{wang2021PtO}
Wang, K., Shah, S., Chen, H., Perrault, A., Doshi-Velez, F., and Tambe, M.
\newblock Learning {MDP}s from features: Predict-then-optimize for sequential decision making by reinforcement learning.
\newblock In \emph{Proceedings of the Conference on Neural Information Processing Systems}, pp.\  8795--8806, 2021.

\bibitem[Zimmer et~al.(2021)Zimmer, Glanois, Siddique, and Weng]{zimmer2021MOMDP}
Zimmer, M., Glanois, C., Siddique, U., and Weng, P.
\newblock Learning fair policies in decentralized cooperative multi-agent reinforcement learning.
\newblock In \emph{Proceedings of the International Conference on Machine Learning}, pp.\  12967--12978, 2021.

\end{thebibliography}
\bibliographystyle{icml2025}

\newpage
\appendix



\section{Theoretical Results}
\label{sec:appendixTheoretical}

For this section, we use an alternate notation, replacing $\beta$ with $\eta= \frac{\beta}{1-\beta_o}$ to make the equations easier to read, such that:
\begin{align*}
    (1-\beta)U + \beta F \Leftrightarrow U + \eta F
\end{align*}

This does not affect the allocation made by the ILP, as it only scales all Q-values by $1/(1-\beta)$. This would only be undefined when $\beta=1$, but we avoid that condition in our proofs. As $\beta\rightarrow1, \eta\rightarrow\infty$, and for any $\beta'>\beta$, $\eta'>\eta$. Note that in the theorem statements, we replace $\beta$ with $\eta$, but the proofs are equivalent.

The following results hold for any fairness function used in the DECAF formulation.

\begin{proposition}
As $\eta_{test}\xrightarrow{} 0$, all fair-only models behave in a utility-maximizing manner.
\end{proposition}
We state this without proof. It is easy to follow how this holds, as at $\eta=0$, the fairness model does not play any role in the decision making.

\setcounter{theorem}{0}
\begin{theorem}
Given perfect estimates for utility and fairness, increasing $\eta$ always improves the one-step fairness gain for SO with $\gamma=0$.
\label{th:theorem_fair_app}
\end{theorem}
\begin{proof}
We assume that the utility and fairness estimators are converged, i.e., the estimates of fairness and utility are correct. For the following discussion, assume the environment has evolved over some time $t$ and is at a state $s_t$. We consider what changes when we change $\eta$ at this state. Variables used henceforth are conditioned on $s_t$, wherever reasonable. We make the conditioning on $s_t$ implicit and do not notate it, to make it easier to read.

With $\gamma=0$, the optimal utility and fairness estimates equal the one-step return, i.e. the change in utility and fairness because of the resulting joint action. Note that these values are not known to the agents prior to the allocation as they depend on the joint actions of all agents, so computing these estimates is not trivial.

Let $U_{tot}(\mathcal{A})$ and $F_{tot}(\mathcal{A})$ be defined as follows, given an allocation $\mathcal{A}$:
\begin{align}
    U_{tot}(\mathcal{A}) &= \sum_{i\in\alpha}U(\mathcal{A}_i)\\
    F_{tot}(\mathcal{A}) &= \sum_{i\in\alpha}F(\mathcal{A}_i)
\end{align}
We remind the reader that $\mathcal{A}_i$ refers to the action assigned to agent $i$ in the allocation $\mathcal{A}$.

Let $\textbf{Z}_t$ represent the agent metrics at time $t$. Further, let $\mathcal{A}^*$ represent the optimal allocation from the ILP with $\eta$ as the trade-off weight. Since $\mathcal{A}^*$ is optimal, it follows that for all other possible allocations $\mathcal{A}_o$: 
\begin{align}
    U_{tot}(\mathcal{A}^*) + \eta F_{tot}(\mathcal{A}^*) &\ge U_{tot}(\mathcal{A}_o) + \eta F_{tot}(\mathcal{A}_o)  \\
    U_{tot}(\mathcal{A}^*) - U_{tot}(\mathcal{A}_o)&\ge  \eta (F_{tot}(\mathcal{A}_o) - F_{tot}(\mathcal{A}^*)) \label{eq:base_comp} 
\end{align}

We are interested in finding what happens to the allocation when we increase $\eta$.
For $\eta'>\eta$, note that the left side of Eq.~\ref{eq:base_comp} remains the same. Since utility estimates are not affected by changing $\eta$, any other allocation $\mathcal{A}_o$ can only be selected over $\mathcal{A}^*$ if the following condition holds:
\begin{align}
    U_{tot}(\mathcal{A}^*) + \eta' F_{tot}(\mathcal{A}^*) &\le U_{tot}(\mathcal{A}_o) + \eta' F_{tot}(\mathcal{A}_o)  \\
    U_{tot}(\mathcal{A}^*) - U_{tot}(\mathcal{A}_o)&\le  \eta' (F_{tot}(\mathcal{A}_o) - F_{tot}(\mathcal{A}^*)) \label{eq:base_comp2} 
\end{align}

Combining Eqs.\ref{eq:base_comp} and \ref{eq:base_comp2}, we get the following:
\begin{align}
    \eta (F_{tot}(\mathcal{A}_o) - F_{tot}(\mathcal{A}^*)) &\le  \eta' (F_{tot}(\mathcal{A}_o) - F_{tot}(\mathcal{A}^*)) \\
    F_{tot}(\mathcal{A}^*)(\eta'-\eta) &\le F_{tot}(\mathcal{A}_o)(\eta'-\eta) \label{eq:base_comp3} 
\end{align}
Since $\eta\ge0$ and $\eta'>\eta$, Eq.~\ref{eq:base_comp3} can only be true if $F_{tot}(\mathcal{A}_o)>F_{tot}(\mathcal{A}_o)$.

Thus, any allocation $\mathcal{A}_o$ that is optimal (and thus selected by the ILP) for $\eta'>\eta$ is guaranteed to have equal or better fairness than the allocation $\mathcal{A}^*$ at $\eta$.
\end{proof}

We also state the corollary to Theorem~\ref{th:theorem_fair_app}.
\begin{corollary}
Given perfect estimates for utility and fairness, decreasing $\eta$ always improves the one-step utility gain for SO with $\gamma=0$.
\label{th:theorem_util_app}
\end{corollary}
The proof follows a similar structure to Theorem~\ref{th:theorem_fair_app}.

While we only prove the behavior for $\gamma=0$, our empirical results show that we can expect similar behavior for long-horizon estimates. For any state, we will select actions that improve fairness in the long run starting from that state as $\eta$ is increased.

We also show the following useful property:
\begin{theorem}
    For a large enough $\eta$, the fairest allocation will be selected with perfect utility and fairness estimators for SO with $\gamma=0$.
\end{theorem}
\begin{proof}
Let $\mathcal{A}_f$ denote the allocation with the largest fairness gain:
\begin{align*}
    \mathcal{A}_f = \argmax_\mathcal{A} F_{tot}(\mathcal{A}) 
\end{align*}

For simplicity, let us assume no two allocations have the same $F_{tot}(\mathcal{A})$. For any other allocation $\mathcal{A}_o$, we have:
\begin{align}
    F_{tot}(\mathcal{A}_f) > F_{tot}(\mathcal{A}_o)  
\end{align}

Then, $\mathcal{A}_f$ will be optimal and selected by the ILP if the following condition holds:
\begin{align}
  U_{tot}(\mathcal{A}_f) + \eta_f F_{tot}(\mathcal{A}_f) &\ge U_{tot}(\mathcal{A}_o) + \eta_f F_{tot}(\mathcal{A}_o)  \\
  \eta_f &\ge \frac{U_{tot}(\mathcal{A}_o) - U_{tot}(\mathcal{A}_f)}{F_{tot}(\mathcal{A}_f)- F_{tot}(\mathcal{A}_o)} 
\end{align}

We can compute an upper bound for $\eta_f$ by considering the range of values that $U_{tot}$ and $F_{tot}$ can take. Let $U_{max}=\max_\mathcal{A}U_{tot}(\mathcal{A})$, and $F_{max}=\max_{\mathcal{A}, \mathcal{A}\ne \mathcal{A}_f }F_{tot}(\mathcal{A})$.

Then, we have the following:
\begin{align}
    \eta_f &\ge \frac{U_{tot}(\mathcal{A}_o) - U_{tot}(\mathcal{A}_f)}{F_{tot}(\mathcal{A}_f)- F_{tot}(\mathcal{A}_o)}\\
    &\le \frac{U_{max} - U_{tot}(\mathcal{A}_f)}{F_{tot}(\mathcal{A}_f)- F_{tot}(\mathcal{A}_o)}\\
    &\le \frac{U_{max} - U_{tot}(\mathcal{A}_f)}{F_{tot}(\mathcal{A}_f)- F_{max}} = \eta_f^u \label{eq:beta_upper}
\end{align}
Eq.~\ref{eq:beta_upper} gives us an upper bound for $\eta_f$. Thus, for all $\eta>\eta_f^u$, $\mathcal{A}_f$ will be the optimal allocation.
\end{proof}
\begin{corollary}
    For a small enough $\eta$, the most utilitarian allocation will be selected with perfect utility and fairness estimators for SO with $\gamma=0$.
\end{corollary}
The proof follows a similar structure to the proof for the previous theorem.

\section{Learning Algorithm}
\label{sec:appendixAlgo}

\begin{algorithm}
\caption{DECAF Algorithm}
\label{alg:main_loop}
\begin{algorithmic}[1]
\State \textbf{Initialize:} agent network $Q_{\theta}$, target network $Q_{\theta'}$
\State \textbf{Initialize:} $\epsilon$ (exploration rate)
\State \textbf{Initialize:} replay buffer $\mathcal{D}$

\For{episode = 1 to $N_{eps}$}
    \State Decay $\epsilon$ according to decay schedule    
    \State RunEpisode($Q_{\theta}$, $\epsilon$, $T$, env, $\mathcal{D}$)

    \If{episode \% k == 0}
        \Comment Run validation with $\epsilon = 0$
        \State RunEpisode($Q_{\theta}$, 0, $\infty$, env, $\mathcal{D}$)
        \Comment Save validation objective
    \EndIf
    \If{episode \% $\tau$ == 0}
        \State $Q_{\theta'}\leftarrow Q_{\theta}$
        \Comment Update target weights
    \EndIf
    
\EndFor

\State \textbf{Load} model with best validation objective value
\State Run $50$ validation episodes using RunEpisode($Q_{\theta}$, $ 0$, $\infty$, env, $\mathcal{D}$)
\State Save validation results

\end{algorithmic}
\end{algorithm}

\begin{algorithm}
\caption{RunEpisode (Executes a single episode )}
\label{alg:episode_loop}
\begin{algorithmic}[1]
\Function{RunEpisode}{$Q_{\theta}$, $\epsilon$, $T$, env, $\mathcal{D}$}
    \State Reset environment, get initial observation $\mathbf{o}_0$
    \For{t = 1 to $N_{steps}$}
        \State Sample a random number $r \in [0, 1]$
        
        \If{$r < \epsilon$} 
            \State $\mathbf{Q}_t = \mathbf{Q}_\text{random}$ \Comment{Random Q-values for exploration}
        \Else
            \State $\mathbf{Q}_t = Q_\theta(\mathbf{o}_t)$ \Comment{Q-values from the agent}
        \EndIf
        
        \State Use ILP to compute optimal action $\mathbf{a}_t$:
        \State \hspace{1em} $\mathbf{a}_t = \text{solve\_ILP}(\mathbf{Q}_t, \text{env.constraints})$
                
        \State Take step in environment: 
        \State \hspace{1em} $(\mathbf{R}_{f,t}, \mathbf{R}_{u,t}, \mathbf{o}_{t+1}) = \text{env.step}(\mathbf{a}_t)$
        
        \State Store transition $(\mathbf{o}_t, \mathbf{a}_t, \mathbf{R}_{u,t}, \mathbf{R}_{f,t}, \mathbf{o}_{t+1})$ in replay buffer $\mathcal{D}$
        
        \If{$t \% T == 0$}
            \State update($Q_\theta$, $Q_\theta'$, $\mathcal{D}, env)$
        \EndIf
    \EndFor
\EndFunction
\end{algorithmic}
\end{algorithm}

\begin{algorithm}[t]
\caption{Update for Joint Optimization}
\label{alg:JOUpdate}
\begin{algorithmic}[1]
\Function{Update}{$Q_{\theta}$, $Q_{\theta'}$, $\mathcal{D}$, env}

    \State Sample a mini-batch of $n$ experiences from replay buffer $\mathcal{D}$
    
    \For{each experience $\langle \mathbf{o}, \mathcal{A}, \mathbf{r}_u, \mathbf{r}_f, \mathbf{o}' \rangle$ in the mini-batch}
        \State Compute Q-values for the successor observation $Q_\theta(\mathbf{o}')$
        
        \State Solve ILP to get the optimal allocation $\mathcal{A}^*$ for the next observation $\mathbf{o}'$:
        \State \hspace{1em} $\mathcal{A}^* = \text{solve\_ILP}(Q_\theta(\mathbf{o}'), \text{env.constraints})$
        
        \State Compute Q-values for $\mathcal{A}^*$ using the target network:
        \State \hspace{1em} $Q_{\theta'}(\mathbf{o}', \mathcal{A}^*)$
        
        \State Compute the target for the TD update:
        \State \hspace{1em} $\text{target} = (1-\beta)  \mathbf{r}_u + \beta  \mathbf{r}_f + \gamma  Q_{\theta'}(\mathbf{o}', \mathcal{A}^*)$
        
        \State Compute the TD loss:
        \State \hspace{1em} $\text{loss} = \left( Q_\theta(\mathbf{o}, \mathcal{A}) - \text{target} \right)^2$
        
        \State Perform gradient descent on the TD loss to update $Q_\theta$
    \EndFor
\EndFunction
\end{algorithmic}
\end{algorithm}

\begin{algorithm}[t]
\caption{Update for Split Optimization}
\label{alg:SOUpdate}
\begin{algorithmic}[1]
\Function{Update}{$Q_\theta$, $Q_{\theta'}$, $\mathcal{D}$, env}
    \State Sample a mini-batch of $n$ experiences from replay buffer $\mathcal{D}$
    \State Unpack $Q_\theta$ into $U_\theta$ and $F_\theta$
    \State Unpack $Q_{\theta'}$ into $U_{\theta'}$ and $F_{\theta'}$
    
    \For{each experience $\langle \mathbf{o}, \mathcal{A}, \mathbf{r}_u, \mathbf{r}_f, \mathbf{o}' \rangle$ in the mini-batch}
        
        \State Compute the combined Q-values for the successor observation:
        \State \hspace{1em} $Q_\theta(\mathbf{o}') = (1 - \beta)  U_\theta(\mathbf{o}') + \beta  F_\theta(\mathbf{o}')$
        
        \State Solve ILP to get the optimal action $\mathcal{A}^*$ for the next observation $\mathbf{o}'$:
        \State \hspace{1em} $\mathcal{A}^* = \text{solve\_ILP}(Q_\theta(\mathbf{o}'), \text{env.constraints})$
        
        \For{model in \{$U, F$\}}
            \If{model is \( U \)}
                \State Set $M_\theta = U_\theta$, $M_{\theta'} = U_{\theta'}$, and $r = \mathbf{r}_u$
            \Else
                \State Set $M_\theta = F_\theta$, $M_{\theta'} = F_{\theta'}$, and $r = \mathbf{r}_f$
            \EndIf
            
            \State Compute the target for the TD update:
            \State \hspace{1em} $\text{target} = r + \gamma  M_{\theta'}(\mathbf{o}', \mathcal{A}^*)$
            
            \State Compute the TD loss:
            \State \hspace{1em} $\text{loss} = \left( M_\theta(\mathbf{o}, \mathcal{A}) - \text{target} \right)^2$
            
            \State Perform gradient descent on the TD loss to update $M_\theta$
        \EndFor
    \EndFor
    
\EndFunction
\end{algorithmic}
\end{algorithm}

Algorithm~\ref{alg:main_loop} shows the overall training loop used for our experiments, with Algorithm~\ref{alg:episode_loop} showing how each episode is executed. We decay epsilon to 0.05 over half of the total number of episodes. $T$, $k$, $\tau$ decide how frequently we learn, validate and update the target model respectively. 
Algorithm~\ref{alg:JOUpdate} and Algorithm~\ref{alg:SOUpdate} detail how the update step is performed for joint and split models. The update for FO is identical to SO, except omitting the update for the utility model.

\subsection{Model Architecture and Training}
All our models use a learning rate of 0.0003 with the Adam optimizer. For all environments except BiasedDM, we train for 1000 episodes, and run validation every 50 steps for model selection. For BiasedDM, we train for 200 episodes, and validate every 20 episodes.

The neural network architecture for all models is the same, with two fully connected hidden layers, of dimension 20, with ReLU activations. The output (1-dimensional) does not have any activation function. We implement our networks using pytorch. 

We use a replay buffer of size 250000, where one experience is a joint transition across all agents. During training, we sample experiences from the replay buffer, and for each experience, we evaluate actions for all agents using the current online network, solving the ILP to get the best joint action. Then, we score the post-decision state for each agent using the target network, and compute the MSE loss between the target value and value estimates of the selected action from the online network.

We ran all our main experiments on a university compute cluster, with each experiment running on a single CPU node with 12GB RAM. Experiment runtime varied with environment choice. Training a single model with one $\beta$ value took between 30 minutes (BiasedDM) and 2 hours (Matthew). Evaluation, as for the generalization experiments, was performed on a 2019 MacBook Pro, where one episode took 2-5 seconds to run, and we bootstrap over 5 runs. 

\section{Environment Details}
\label{sec:appendix1}
Here, we provide further details about the environments for reproducibility. We will also make the environment and training code available upon acceptance.

\begin{figure*}[t]
    \centering
    \begin{subfigure}[b]{0.19\linewidth}
        \centering
        \includegraphics[width=\linewidth]{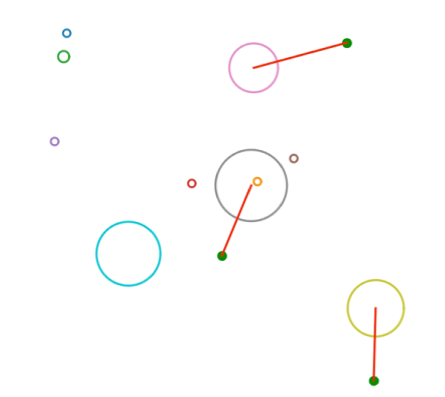}
        \caption*{Matthew}
    \end{subfigure}
    \begin{subfigure}[b]{0.19\linewidth}
        \centering
        \includegraphics[width=\linewidth]{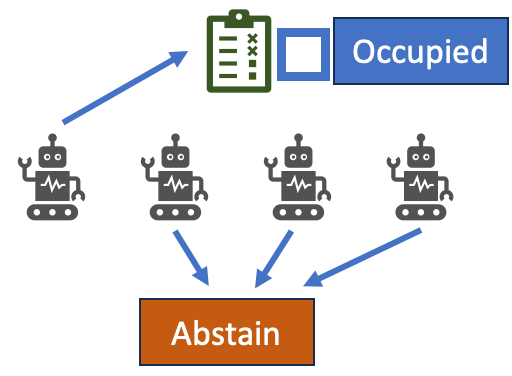}
        \caption*{JobAlloc}
    \end{subfigure}
    \begin{subfigure}[b]{0.19\linewidth}
        \centering
        \includegraphics[width=\linewidth]{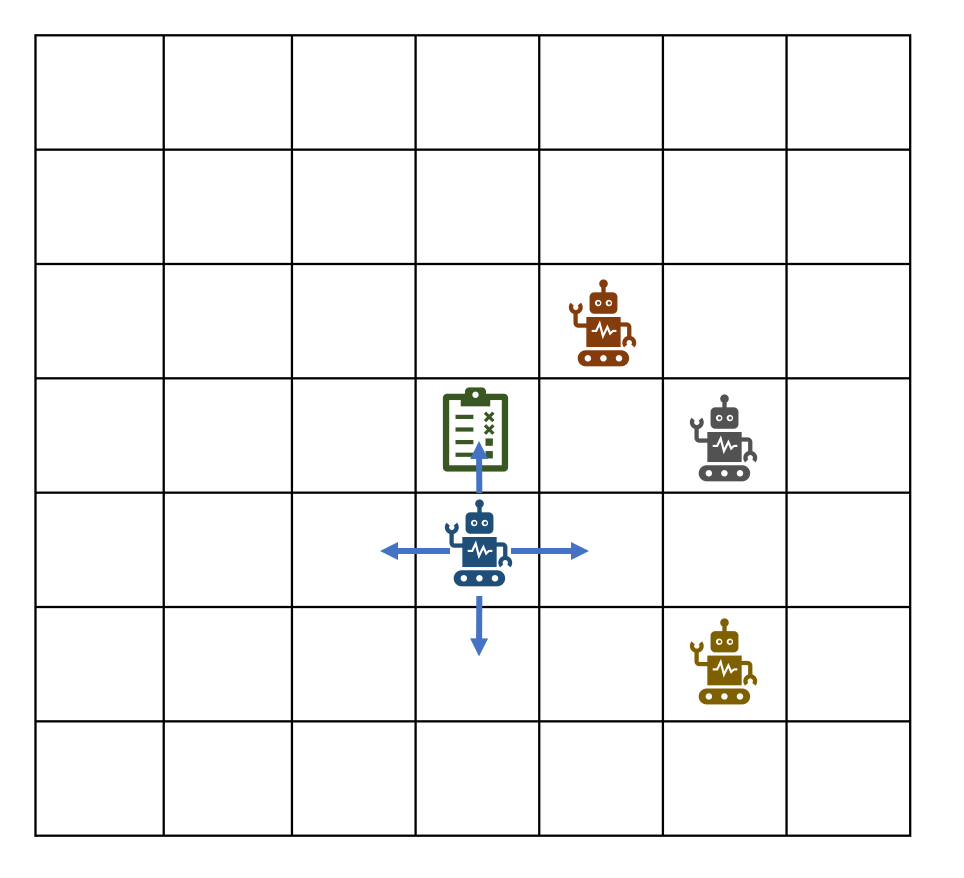}
        \caption*{Job}
    \end{subfigure}
    \begin{subfigure}[b]{0.19\linewidth}
        \centering
        \includegraphics[width=\linewidth]{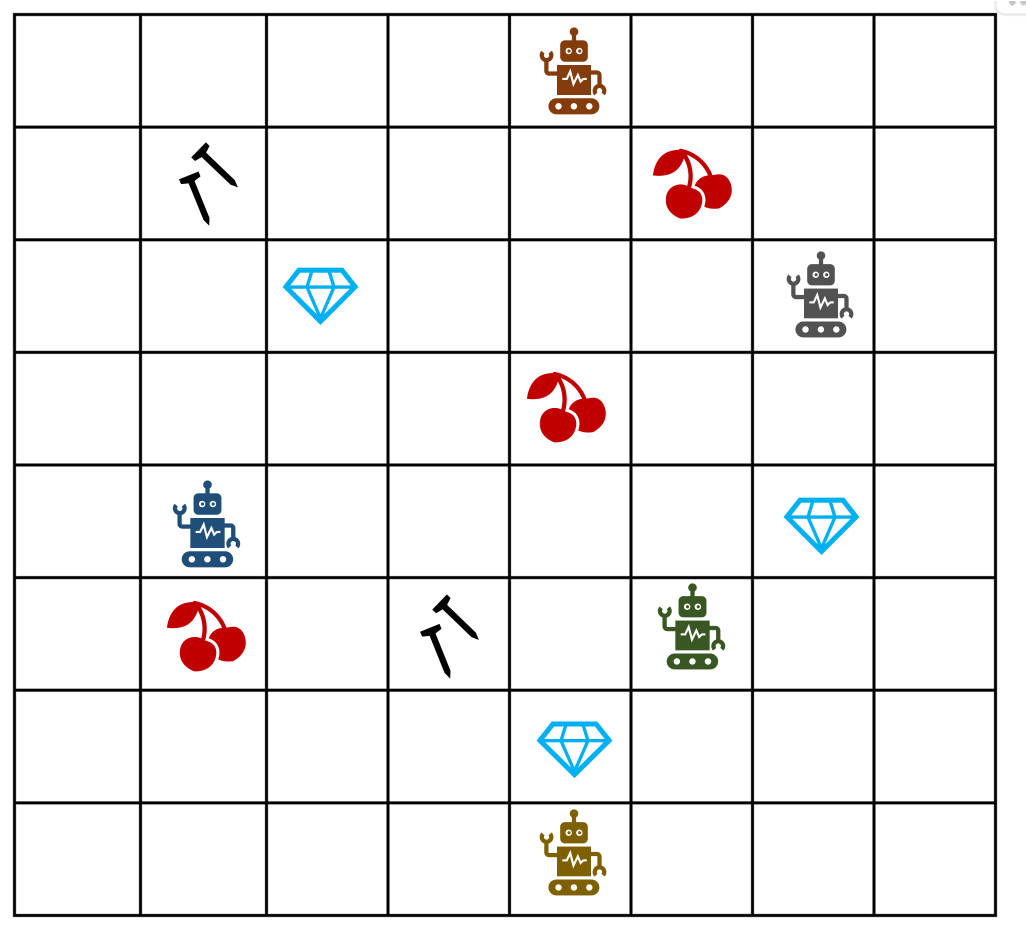}
        \caption*{Plant}
    \end{subfigure}
    \begin{subfigure}[b]{0.19\linewidth}
        \centering
        \includegraphics[width=\linewidth]{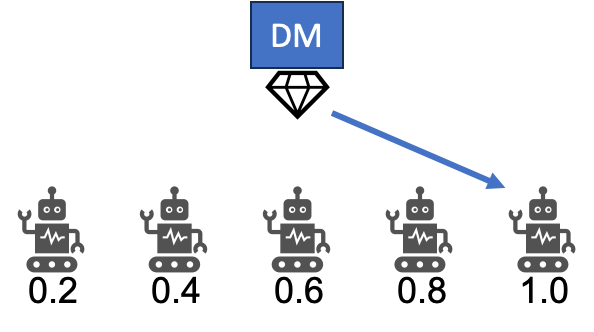}
        \caption*{BiasedDM}
    \end{subfigure}
    \caption{Illustration of all five environments}
    \label{fig:env_illustration}
\end{figure*}

\subsection{Matthew}
One episode for this environment lasts 200 steps. At each step, 10 agents and 3 resources are available on the map. Agent speeds grow proportionally to their size, and a size ceiling exists to prevent agents from growing too large for the environment bounds. Agent and resource positions are 2-D coordinates in $[0,1]$. At the beginning of each episode, agent and resource positions are randomly initialized.
To show the matthew effect, 4 agents are initialized to have a larger initial size than other agents. This allows them to reach resources faster. Agents receive a unit reward when they collect a resource, in addition to a small increase in size and speed. Resources allocated to agents are reserved, and other agents cannot pick them up. A new resource only spawns when an agent reaches its allocated resource, not when it is allocated.

\subsection{Job}
One episode for this environment lasts 100 steps. There are 4 agents on a $7\times 7$ grid, and agents can move in any cardinal direction or stay in their location. Agents cannot occupy the same location as any other agent, and they receive a unit reward when occupying the resource location. Attempting to move out of the grid results in a no-op. The agents start off at the corners of the grid. The job's location is fixed at the center of the grid, and its position remains fixed for the entire episode.
The objective in this environment is for agents to learn to share the job instead of occupying it alone.

\subsection{JobAlloc}
This is a simplified version of the Job environment, after removing the grid and casting it more directly as a resource allocation problem. We still maintain the core challenge of agents learning to give up resources at some point so that other agents can also occupy the job, by adding a constraint that agents can only occupy a resource if the resource is free at the beginning of the timestep. This requires a joint action where all agents decide to abstain, in order to change the agent occupying the resource. One episode lasts 100 steps, with 4 agents.

\subsection{Plant}
One episode for this environment lasts 200 steps. There are 5 agents on a $8 \times 8$ grid, where agents can move in cardinal directions.  The grid also contains 8 resources, which can be of three different types. Each agent gets a reward when they construct a `unit'. Each agent has a requirement of a set of resources it must collect  so that it can construct this unit. The requirements are: 
\[
\{ (2, 1, 0), (1, 0, 1), (1, 0, 0), (1, 3, 0), (0, 1, 2) \}
\]
For example, agent 1 requires two resources of type 1 and one resource of type 2 following which it can get a reward. Some agents are given easier requirements to fulfill, which creates a bias in the number of units agents produce. The resources and agent locations are randomly initialized at the beginning of each episode. The allocation follows a similar process to the Matthew environment, where actions are allocations of agents to resources, and other agents cannot pick up resources already allocated to other agents. When a resource is picked up, another resource of the same type appears in a random location on the map.

\subsection{BiasedDM}
One episode for this environment lasts 100 steps. At each time step, the decision maker allocates one resource to one of five agents. The utility of assigning the resource to each agent is different for the decision maker. In other environments, fairness is computed as the variance over the accumulated rewards for each agent. In this environment, however, fairness is computed over the resource rate, which is the fraction of steps in which an agent received the resource ($z_i\in[0,1]$).
\begin{align}
    z_i = \frac{\text{Num. resources}}{time}
\end{align}

This also results in much smaller variances, thus our hyperparameter search for this domain explores the higher range of $\beta$ values more.

\section{Learning with Other Fairness Functions}
\begin{figure*}
    \centering
    \includegraphics[width=0.98\linewidth]{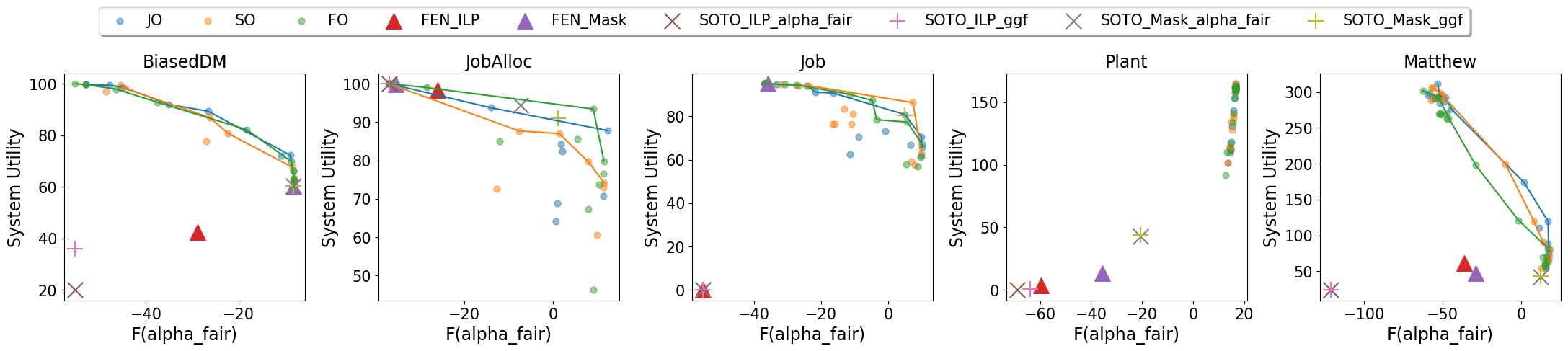}
    \includegraphics[width=0.98\linewidth]{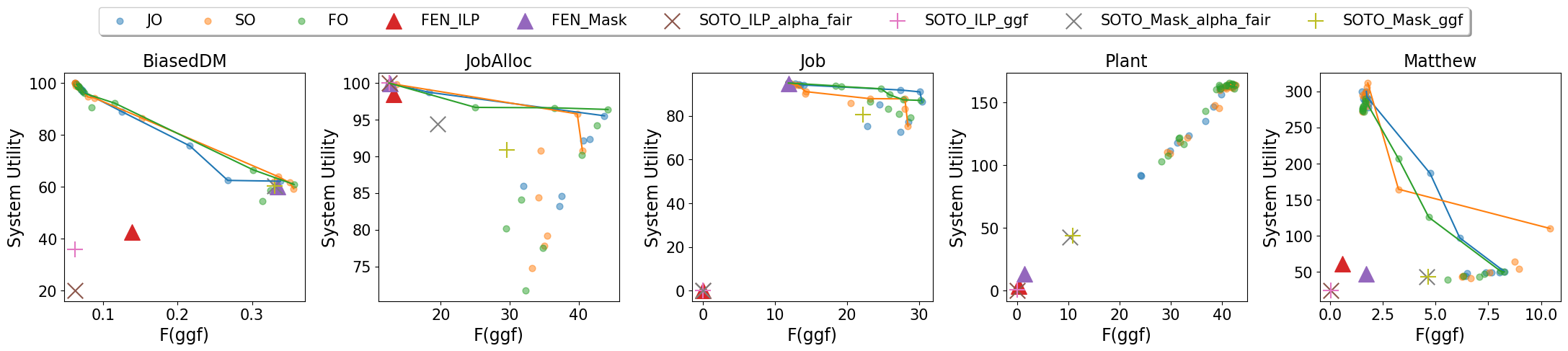}
    \includegraphics[width=0.98\linewidth]{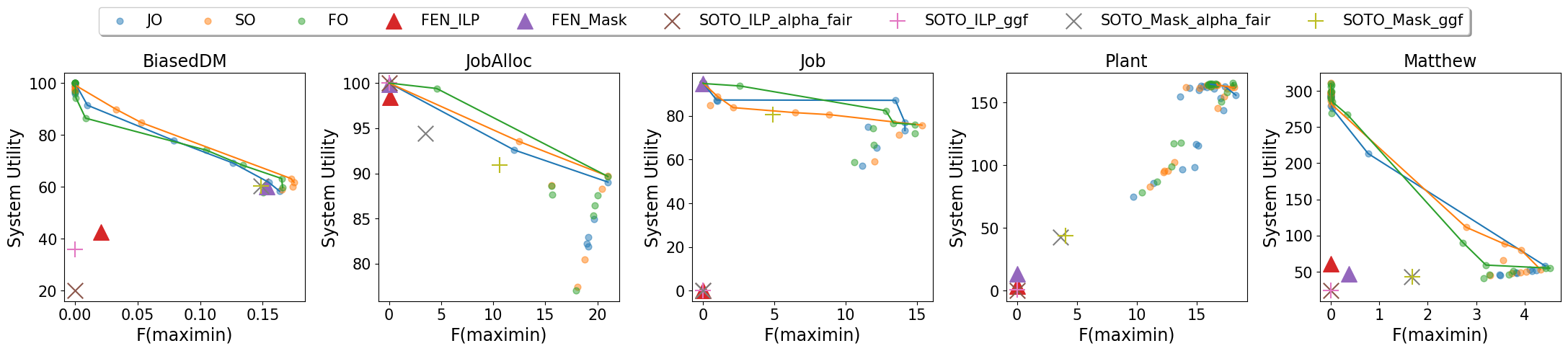}
    \caption{Results training DECAF with $\alpha$-fairness, GGF and maximin fairness functions. The lines show the Pareto fronts for each model type.}
    \label{fig:other_metrics}
\end{figure*}

As mentioned in the main text, we are not restricted to using variance and the given decomposition. We show here results for three more functions: $\alpha$-fairness, $GGF$, and  maximin.

\paragraph{1. \textbf{$\alpha$-fairness}}
The $\alpha$-fair function can be stated as:
\begin{align}
    F_\alpha(\textbf{Z}) =& \sum_{z_i\in\textbf{Z}}\begin{cases}
        \frac{z_i^{1-\alpha}}{1-\alpha} & \alpha\neq 1 \\
        \log z_i & \alpha=1
    \end{cases}
\end{align}
With $\alpha=1$, this is equivalent to proportional fairness or log Nash Welfare, both popular notions of fairness, while at $\alpha=0$ it represents the utilitarian objective. In our experiments, we use $\alpha=1$.

\paragraph{2. \textbf{Generalized Gini Function (GGF)}}
Given a sequence of positive, fixed, strictly decreasing weights $\mathbf{w}$, the GGF function can be stated as:
\begin{align}
    G_w(\textbf{Z}) = \sum_{i\in\alpha} \mathbf{w}_i z_i^{\uparrow}
\end{align}
Here, $\textbf{z}^{\uparrow}$ represents the vector obtained by sorting the $\textbf{Z}$ vector. This function can also represent a diverse set of SWFs including utilitarian and maximin fairness. In our experiments, we use decreasing negative powers of 2 as the weights (i.e. $\mathbf{w}=[1, 2^{-1}, 2^{-2}, \dots 2^{-(n-1)}]$). 

\medskip\noindent For both of the above objectives, we use the equal decomposition, where the fairness reward is computed as an equal division of the change in the metric value across all agents.
\begin{align}
    R_f(\mathbf{s}_t, \mathcal{A}^t) = \left[ \frac{\Delta \mathcal{F}|\mathcal{A}^t}{n} \right]_{i\in\alpha}
\end{align}

\paragraph{3. \textbf{Maximin Fairness}}
This function captures the worst off utility of any agent:
\begin{align}
    F_{MMF}(\textbf{Z}) = \min(\textbf{Z})
\end{align}
This is a hard objective to learn, as the maximin objective changes only when the worst-off agent is improved. 
We decompose this reward by combining the global signal with the per-agent contribution towards improving the minimum. Intuitively, each agent receives a reward for a joint action that improves the minimum, but the agents that were at the minimum (and improved) receive a larger reward. 
\begin{align}
    r_{f,i} &= \frac{\min(\textbf{Z}') - \min(\textbf{Z})}{n} \\
    r_{f,i} &= r_{f,i} +
    \begin{cases}
        z'_i - z_i & \text{if } z_i =\min(\textbf{Z}) \\
        0 & \text{otherwise}
    \end{cases}\\
    r_{f,i} &= r_{f,i} + 
    \begin{cases}
        z'_i - z_i & \text{if } z'_i =\min(\textbf{Z}') \\
        0 & \text{otherwise}
    \end{cases}\\
    R_{f,i} &= \frac{r_{f,i}}{\sum_j r_{f,j}} \left(\min(\textbf{Z}') - \min(\textbf{Z})\right)
\end{align}

\begin{figure*}[t]
    \centering
    \includegraphics[width=0.19\linewidth]{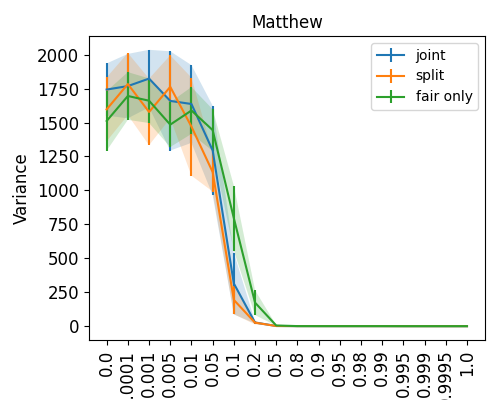}
    \includegraphics[width=0.19\linewidth]{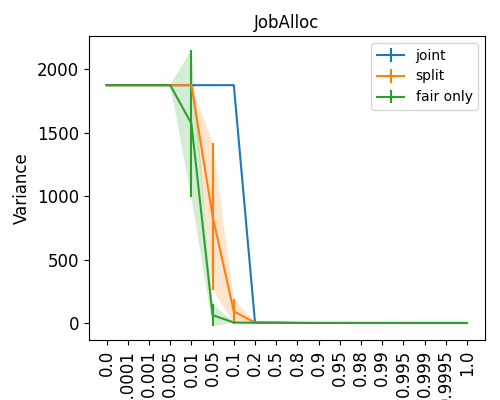}
    \includegraphics[width=0.19\linewidth]{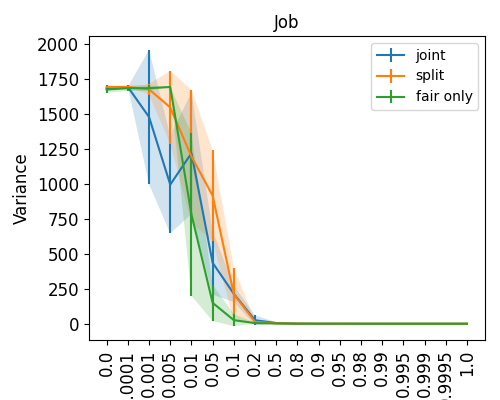}
    \includegraphics[width=0.19\linewidth]{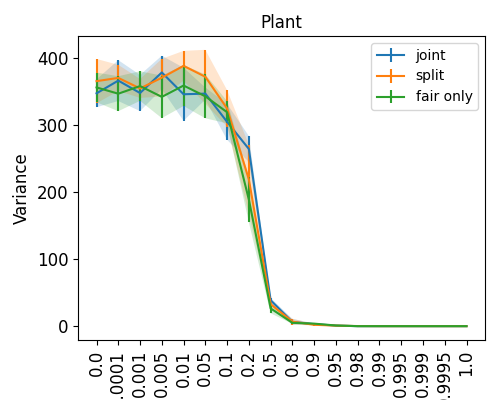}
    \includegraphics[width=0.19\linewidth]{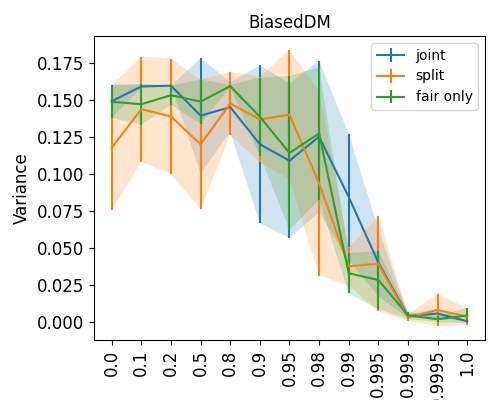}
    
    \includegraphics[width=0.19\linewidth]{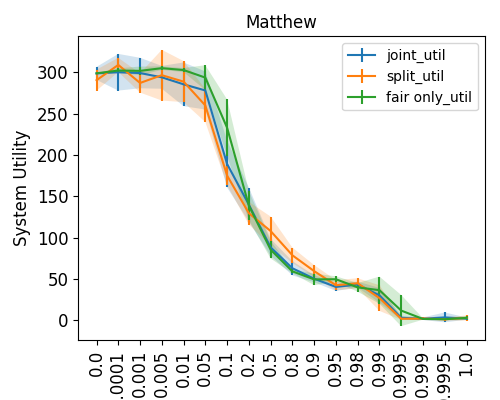}
    \includegraphics[width=0.19\linewidth]{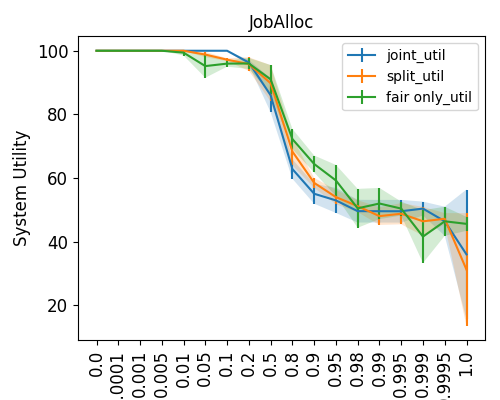}
    \includegraphics[width=0.19\linewidth]{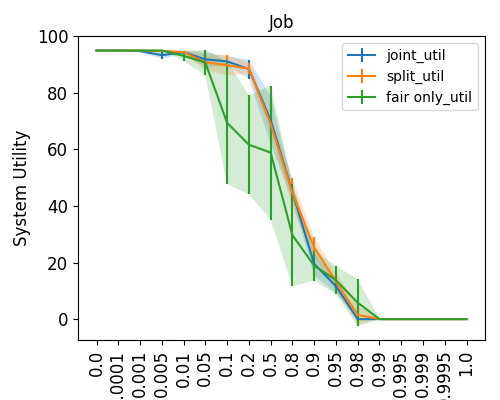}
    \includegraphics[width=0.19\linewidth]{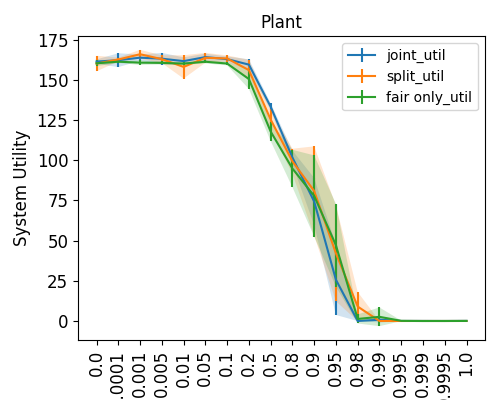}
    \includegraphics[width=0.19\linewidth]{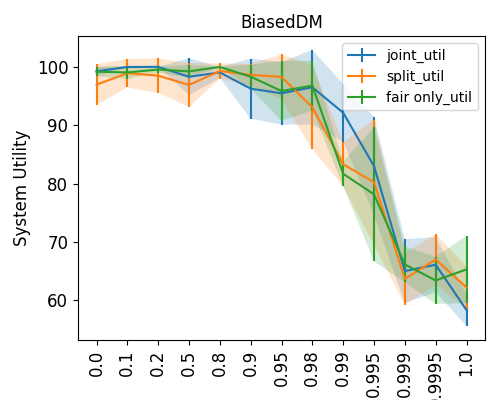}
    \caption{Effect of changing $\beta$ on variance (top row) and utility (bottom row) for all three methods on all five environments. 
    The shaded area shows the 1-$\sigma$ error bar. 
    We observe that the performance of FO has large variation for intermediate $\beta$ values.}
    \label{fig:error-bars}
\end{figure*}

\subsection{DECAF Results with Other Fairness Metrics}
Figure~\ref{fig:other_metrics} shows the results of our approaches when using different fairness functions, with the decompositions as described above.  In general, we observe our methods offer a good range of trade-offs in all environments, often Pareto-dominating SOTO and FEN. We describe some additional details about these results in this section.

In almost all environments, it is possible to get significant fairness improvement starting from the utilitarian model. The only exception is the Plant environment for $\alpha$-fair and GGF, where the utilitarian solution is almost optimal for fairness as well. In Job and JobAlloc, SOTO masked models are competitive, especially for the $\alpha$-fair fairness function, while some DECAF models are not as performant. This may suggest the existence of a more informative decomposition for this function which can lead to better learning. In BiasedDM, SOTO and FEN face a significant disadvantage, as they can not contend with the misaligned fairness signal. Thus, they can either know the system utility, or the payoffs on which fairness is based. In the prior case, the learned fair policy performs poorly for both utility and fairness, as it tries to equalize the accumulated value from the decision maker's perspective, or, as in the latter case (selected for the results shown), only looks at the resource distribution and has no idea of the utility to the decision maker.

$\alpha$-fair and GGF could also be decomposed to compute per-agent contributions, but since the metrics treat all agents indepependently, the interaction between agent utilities is harder to learn, especially for the Job and JobAlloc environment, where being fair requires a globally suboptimal decision in terms of utility. In other environments, the pressure for fairness is better captured in the ILP optimization, by the valuations of other agents that could benefit more from getting certain resources.

\begin{figure*}[t]
    \centering
    \subfloat[Matthew]{
    \includegraphics[width=0.19\linewidth]{Figures/Generalization/Approx_pareto/Matthew/Split_ApproxPareto.png}
    \label{fig:matthew_split}}
    \subfloat[JobAlloc]{
    \includegraphics[width=0.19\linewidth]{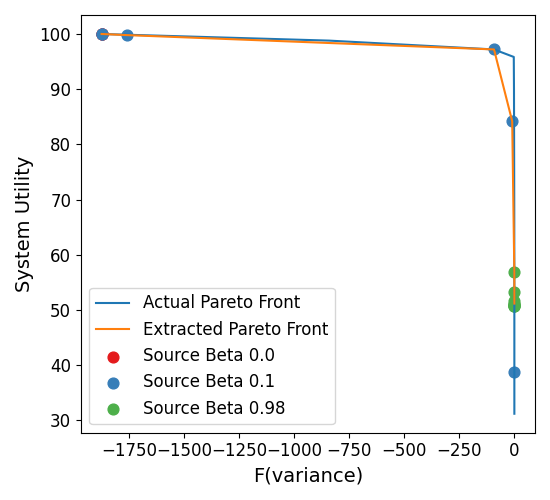}
    \label{fig:joballoc_split}}
    \subfloat[Job]{
    \includegraphics[width=0.19\linewidth]{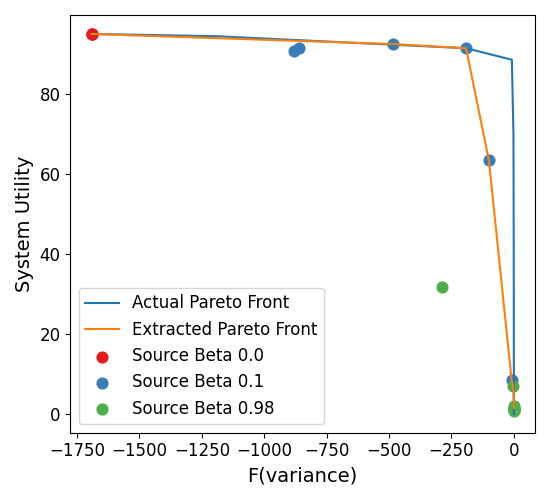}
    \label{fig:job_split}}
    \subfloat[Plant]{
    \includegraphics[width=0.19\linewidth]{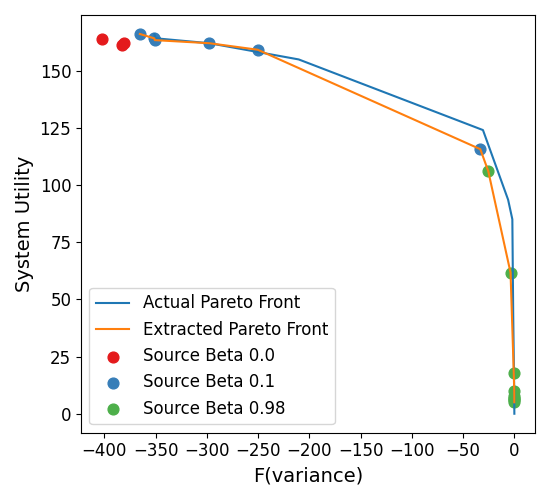}
    \label{fig:plant_split}}
    \subfloat[BiasedDM]{
    \includegraphics[width=0.19\linewidth]{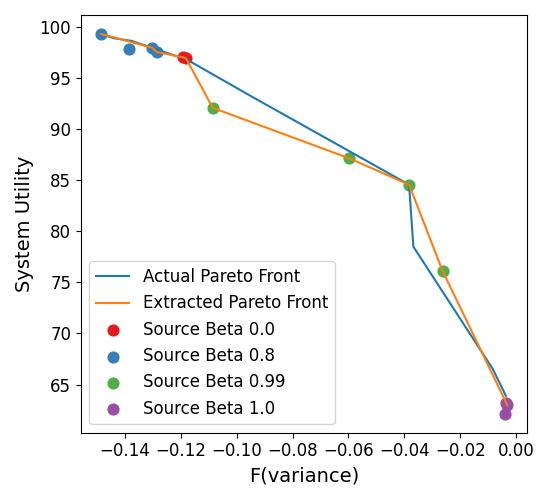}
    \label{fig:biaseddm_split}}
    \\
    \includegraphics[width=0.19\linewidth]{Figures/Generalization/Approx_pareto/Matthew/SplitNoUtility_ApproxPareto.png}
    \includegraphics[width=0.19\linewidth]{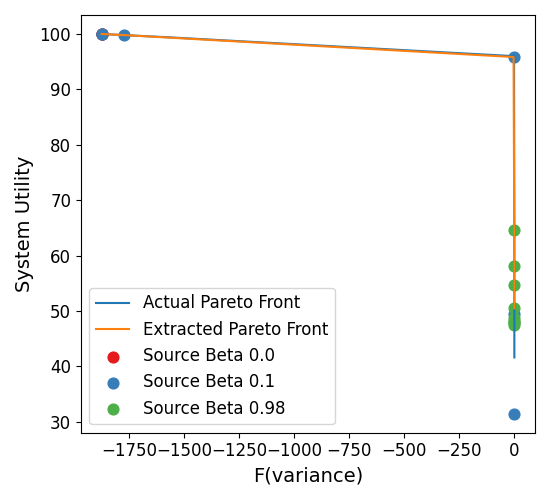}
    \includegraphics[width=0.19\linewidth]{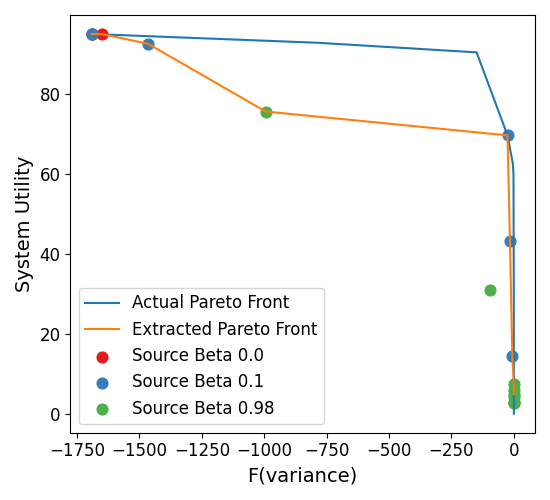}
    \includegraphics[width=0.19\linewidth]{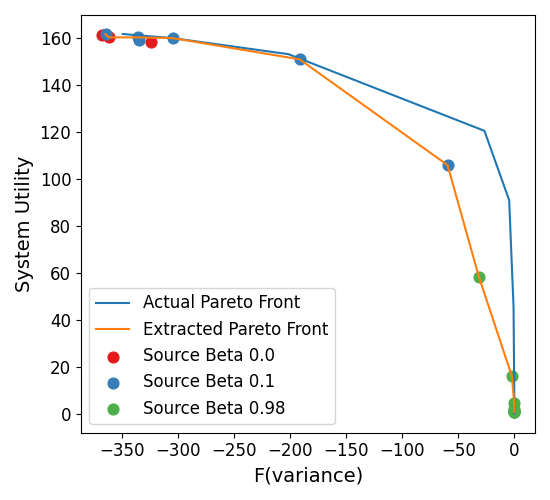}
    \includegraphics[width=0.19\linewidth]{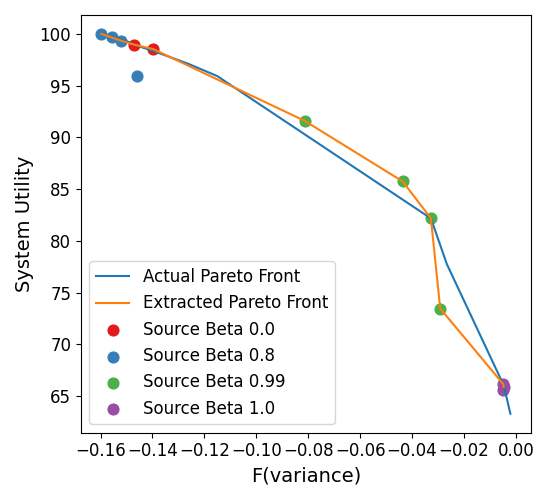}
    \caption{Approximate Pareto fronts for Split Optimization and Fair-Only Optimization models for variance across different environments. The top row represents SO models, while the bottom row represents FO models.}
    \label{fig:approx_pareto_fronts}
\end{figure*}

\subsection{On the Importance of Past Discounting and Warm Starts}
We also highlight two departures from prior work in our implementation of DECAF for learning fairness: (1) We discount the agent metrics $\textbf{Z}$ over the past, and (2) we implement warm starts for initializing $\textbf{Z}$ instead of initializing at 0. The past discounts are necessary to account for the time dependence of various normalized fairness metrics, like the Gini coefficient and coefficient of variation, or the variance over normalized agent metrics. In these cases, actions taken early on can cause larger changes to the fairness metric while actions taken after a sufficient history has been established have an imperceptible effect. This is undesirable in long-horizon settings, but can be remedied by past discounts, effectively `forgetting' events that happened far in the past. The warm starts function to counteract the other pitfall in some fairness metrics: The zero vector is perfectly fair, and any changes can incur a large fairness penalty. Adding randomized warm starts helps in preventing the algorithm from converging to this trivial solution. In practice, we keep the warm start values small, so that the past discounts effectively scale them down to zero over the course of an episode.

\subsubsection{Past Discounts and Warm Starts}
Past discounts and warm starts significantly helped in improving the stability of learning in our experiments, especially with variance as the fairness function. Small initial perturbations (that are smoothed out over time using past discounts) help in exploration of different states, as well as in avoiding the cold start problem where any action would lead to a worse state and hence agents learn to avoid taking beneficial actions. For variance, we used warm starts based on the size of the maximum rewards possible in an episode in each environment. For BiasedDM, the warm starts are used to create an initial 'resource rate' by normalizing based on the number of total warm start resources, as this environment uses resource rates as the payoff vector $\textbf{Z}$ . 

For $\alpha$-fair, we used $\alpha=1$ for the experiments, and hence we avoided using warm starts for this metric, as the log function is sensitive to small perturbations especially with near-zero utilities. For GGF, we used a warm start of 0.1 for each environment. For both of these functions, we used no past discounting, as the metrics are additive functions over agent utilities, so any past discounting would cause negative fairness gains. 
For maximin, we used the same past discounting and warm start values as variance. The warm start and past discount values for all environments are given in Table~\ref{tab:warm-past-disc}.

\begin{table}[]
\caption{Warm start ($w$) and past discount ($\gamma_p$) values for different environments and fairness functions used for DECAF.}
\label{tab:warm-past-disc}
\begin{minipage}{\linewidth}
\centering
\resizebox{\linewidth}{!}{%
\begin{tabular}{|c|c|c|c|c|c|c|}
\hline
\multicolumn{1}{|l|}{}                  &                     & \textbf{Matthew} & \textbf{Plant} & \textbf{Job} & \textbf{JobAlloc} & \textbf{BiasedDM} \\ \hline
\multirow{2}{*}{\textbf{$\alpha$-fair}} & \textbf{$w$}        & 0                & 0              & 0            & 0                 & 0                 \\ \cline{2-7} 
                                        & \textbf{$\gamma_p$} & 1                & 1              & 1            & 1                 & 1                 \\ \hline
\multirow{2}{*}{\textbf{GGF}}           & \textbf{$w$}        & 0.1              & 0.1            & 0.1          & 0.1               & 0.1               \\ \cline{2-7} 
                                        & \textbf{$\gamma_p$} & 1                & 1              & 1            & 1                 & 1                 \\ \hline
\multirow{2}{*}{\textbf{Maximin}}       & \textbf{$w$}        & 5                & 1              & 3            & 3                 & 2                 \\ \cline{2-7} 
                                        & \textbf{$\gamma_p$} & 0.995            & 0.995          & 0.995        & 0.995             & 0.999             \\ \hline
\multirow{2}{*}{\textbf{Variance}}      & \textbf{$w$}        & 5                & 1              & 3            & 3                 & 2                 \\ \cline{2-7} 
                                        & \textbf{$\gamma_p$} & 0.995            & 0.995          & 0.995        & 0.995             & 0.999             \\ \hline
\end{tabular}%
}
\end{minipage}
\end{table}

Given a warm start value of $w$, we compute an initial distribution of pseudo-resources by uniformly sampling from a region of width $w/4$ centered around $w$. For additive utilities, we use past discounts to decay the accumulated payoffs in $\textbf{Z}$ before adding the current time-step's reward. If $\gamma_p$ is the past discount factor, and $R_{t,i}$ is the resource value allocated to agent $i$ at time $t$, then we compute:
\begin{align*}
    z_i^{t+1} = \gamma_p z_i + R_{t,i}
\end{align*}

For averaged rewards (as in BiasedDM), where the payoff $z_i= {\#resources}/{\#timesteps}$, we compute the discount by reweighting both the numerator and denominator. Note that this can be easily extended to act as a `resource rate' where the denominator is the number of potential resources the agent could have received, instead of the time.
\begin{align*}
    z_i^t &= \frac{res_i}{t_i}\\
    res_i &= \gamma_p res_i + R_{t,i} \\
    t_i &= \gamma_p t_i + 1 \\
    z_i^{t+1} &= \frac{res_i}{t_i}
\end{align*}

\begin{figure*}[t]
    \centering
    \subfloat[Matthew]{
    \includegraphics[width=0.32\linewidth]{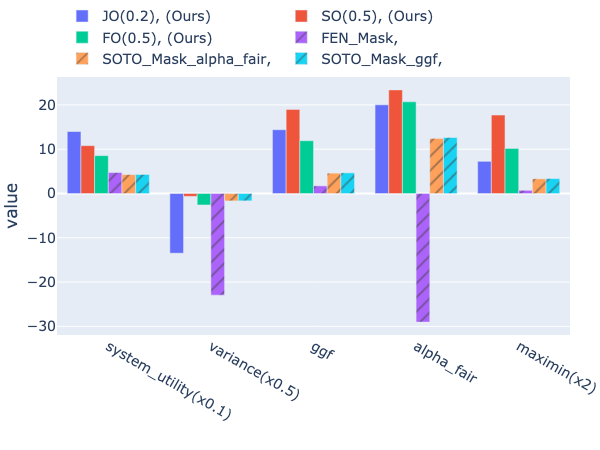}
    }
    \subfloat[JobAlloc]{
    \includegraphics[width=0.32\linewidth]{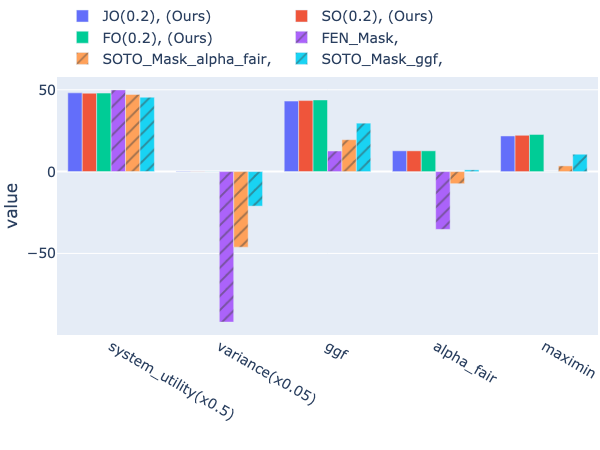}
    }
    \subfloat[Job]{
    \includegraphics[width=0.32\linewidth]{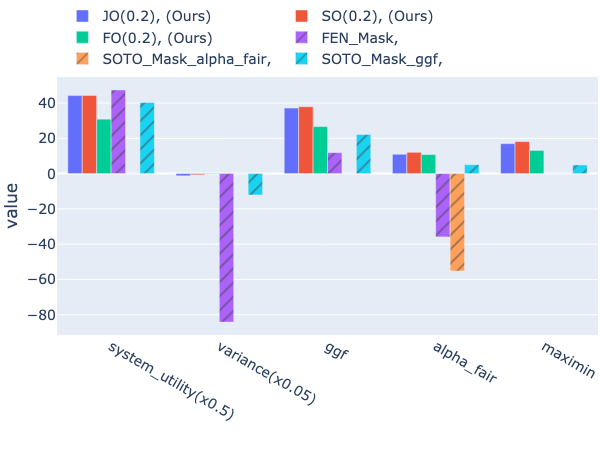}
    }
    
    \subfloat[Plant]{
    \includegraphics[width=0.32\linewidth]{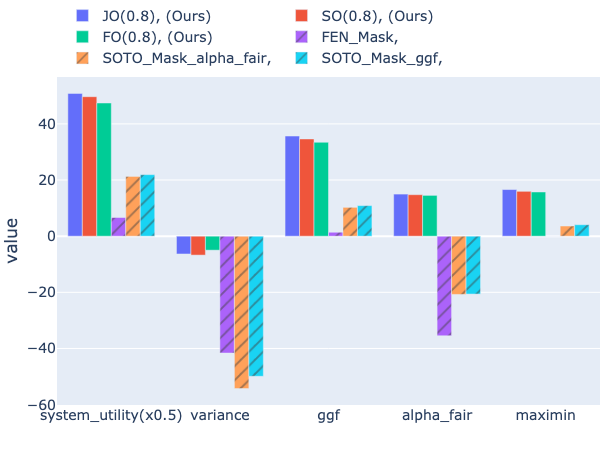}
    }
    \subfloat[BiasedDM]{
    \includegraphics[width=0.32\linewidth]{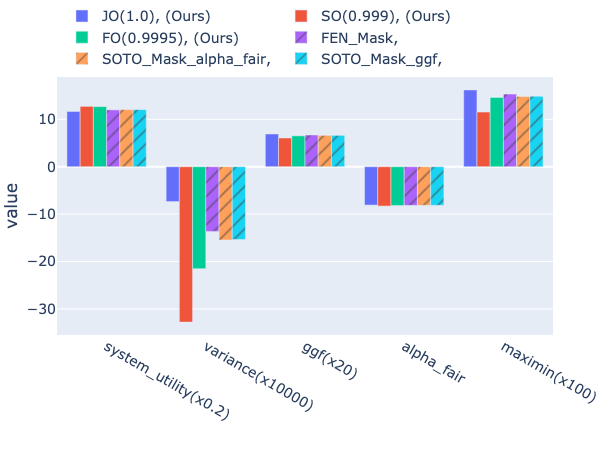}
    }
    \caption{Comparisons of selected DECAF models against the three baselines, scaled to fit on the same axes. We omit results for ILP versions of FEN and SOTO due to their poor performance. We selected DECAF models (trained on variance) that maximized $0.1 U - 0.9\var(\textbf{Z})$. The numbers in brackets denote the selected $\beta$ value for our models.}
    \label{fig:all_barplots}
\end{figure*}

\section{FEN and SOTO Implementation Details}
For FEN and SOTO, we use 5 times the number of training steps as used for DECAF, to allow the PPO based approaches sufficient trajectories to learn from, and to better match the experiments in the respective papers. We do not use past discounts and warm starts.
For BiasedDM, we chose the reward vector to be the number of resources each agent received (1 for each agent), instead of the biased utility to the decision-maker ($0.2\cdot i$ for agent $i$). Since our fairness metrics operate on the vector of resources as well, and since FEN and SOTO aim to learn fairness, this was the obvious choice. We also ran experiments where the reward vector was based on the decision-maker's biased utility, and found the solutions to be poor for both utility and fairness based on resources, so we omit them in our results. However, the greedy self-oriented model in SOTO was trained using the decision-maker's reward.

We used FEN without gossip, where the agent is directly communicated the distribution information, without need for inter-agent communication rounds. This is expected to be the stronger version of FEN. 
For all models, we used the same features from the DECA environments, containing the agent's relative advantage as a feature. In addition to this, SOTO also required the entire payoff distribution $\textbf{Z}$ and information about nearby neighbors for the tiered team-oriented network.  

The Job environment uses a shaping reward for the distance to the job as a scaled penalty. For the Job environment for SOTO, we reduced the weight of the shaping reward to 0.01 to minimize its effect on the optimization. We found this to be the best setting for learning in this environment. The ILP version of SOTO was not able to learn at all in this setting, even when we completely removed the shaping reward. Again, the inability to use shaping rewards is a significant handicap for methods like SOTO, since they optimize fairness over the rewards. Our methods explicitly decouple learning utility and fairness, so they are more expressive, and able to learn better especially at intermediate values of $\beta$.

\section{Extended Results for variance}

\begin{figure*}[t]
    \centering
    \subfloat[SO models]{
        \begin{minipage}{0.48\linewidth}
            \centering
            \includegraphics[width=0.48\linewidth]{Figures/Generalization/Matthew/Split_system_utility.png}
            \includegraphics[width=0.48\linewidth]{Figures/Generalization/Matthew/Split_fairness.png}
            \\
            \includegraphics[width=0.48\linewidth]{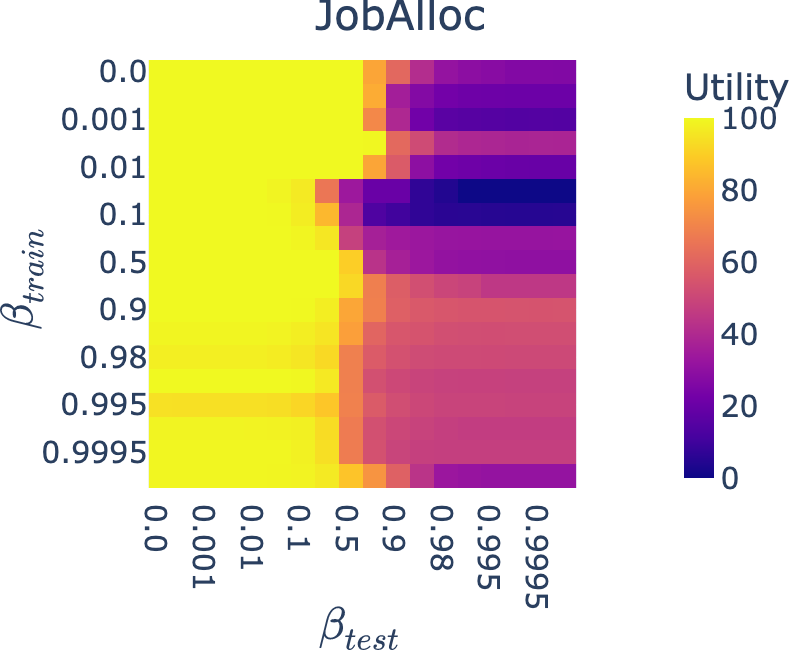}
            \includegraphics[width=0.48\linewidth]{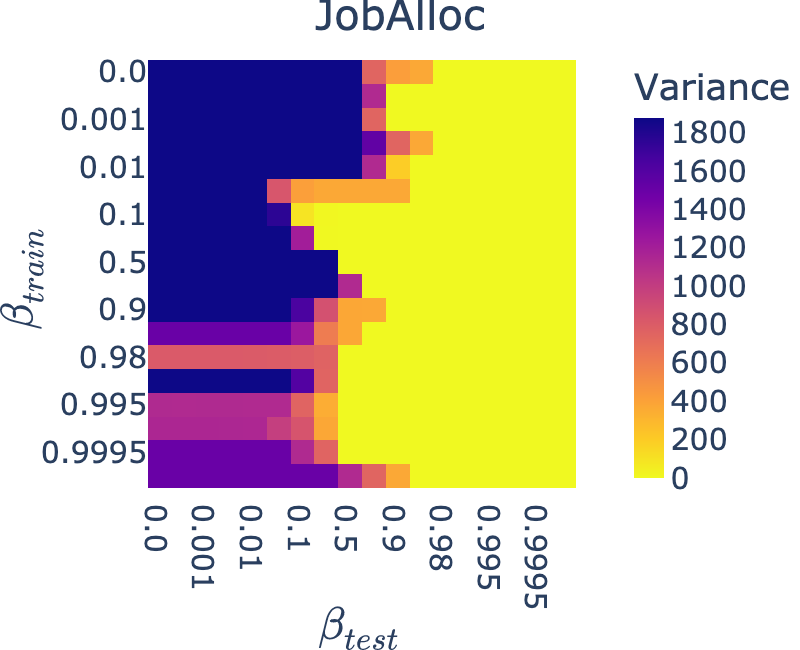}
            \\
            \includegraphics[width=0.48\linewidth]{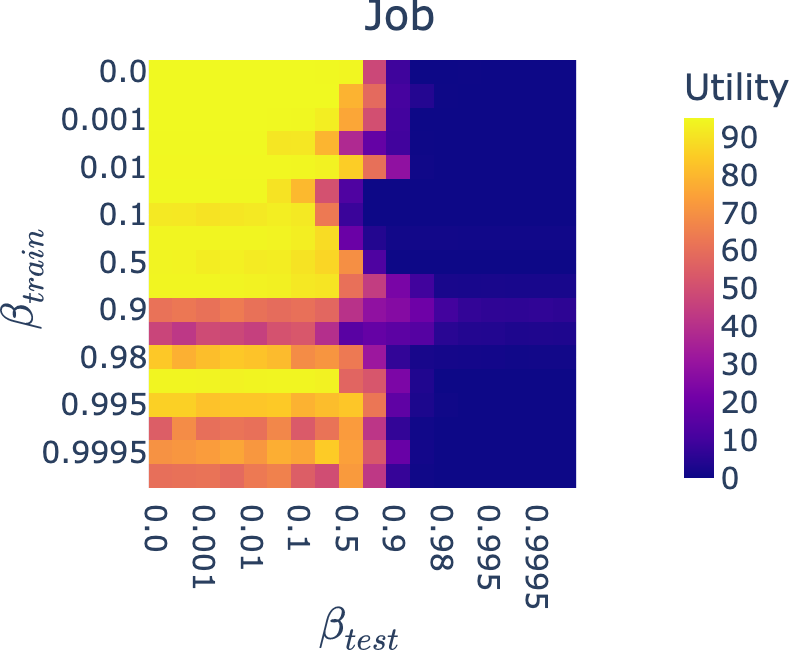}
            \includegraphics[width=0.48\linewidth]{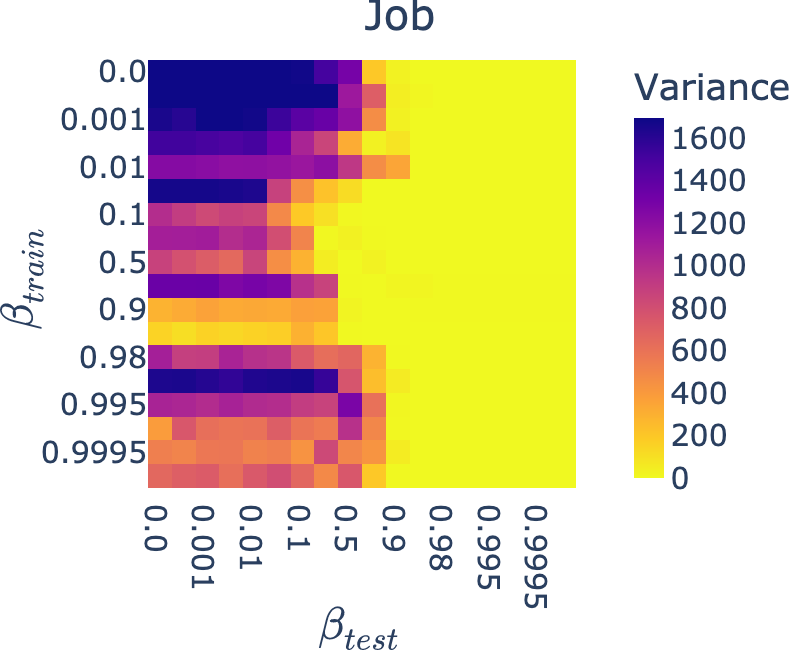}
            \\
            \includegraphics[width=0.48\linewidth]{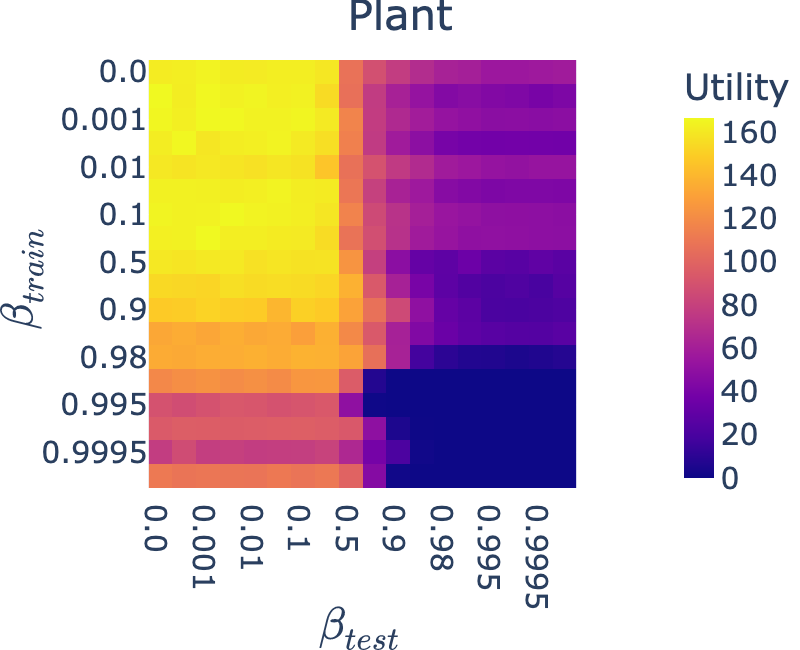}
            \includegraphics[width=0.48\linewidth]{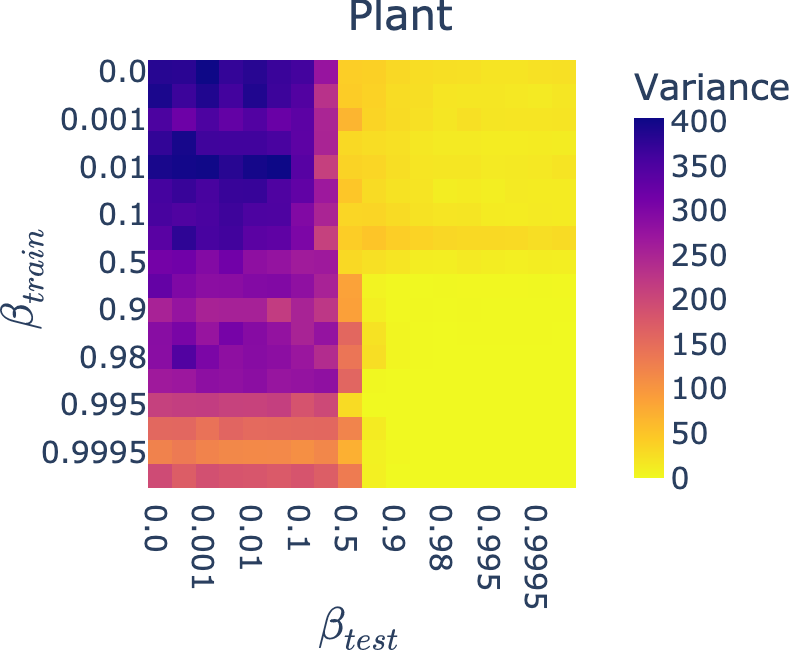}
            \\
            \includegraphics[width=0.48\linewidth]{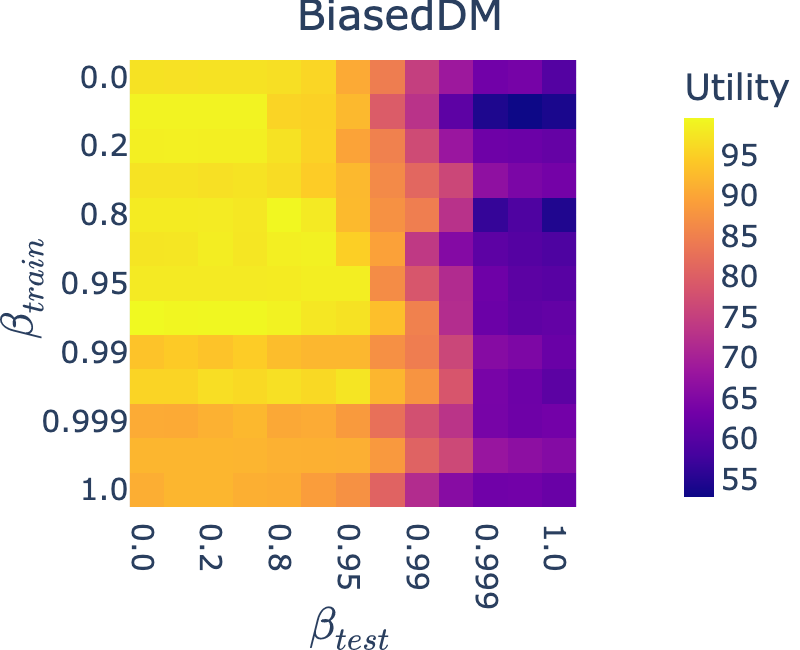}
            \includegraphics[width=0.48\linewidth]{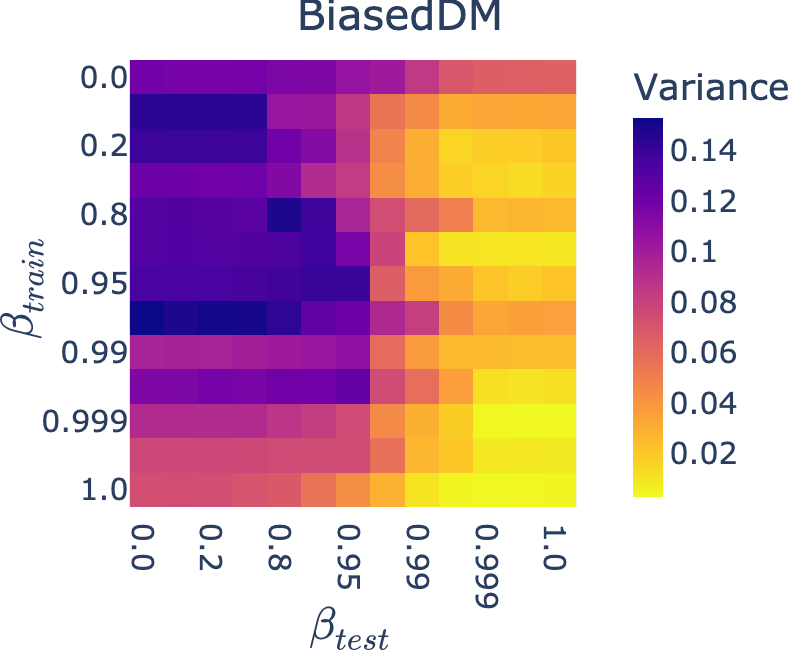}
        \end{minipage}
        \label{fig:all_results_gen_SO}
    }
    \subfloat[FO models]{
        \begin{minipage}{0.48\linewidth}
            \centering
            \includegraphics[width=0.48\linewidth]{Figures/Generalization/Matthew/SplitNoUtility_system_utility.png}
            \includegraphics[width=0.48\linewidth]{Figures/Generalization/Matthew/SplitNoUtility_fairness.png}
            \\
            \includegraphics[width=0.48\linewidth]{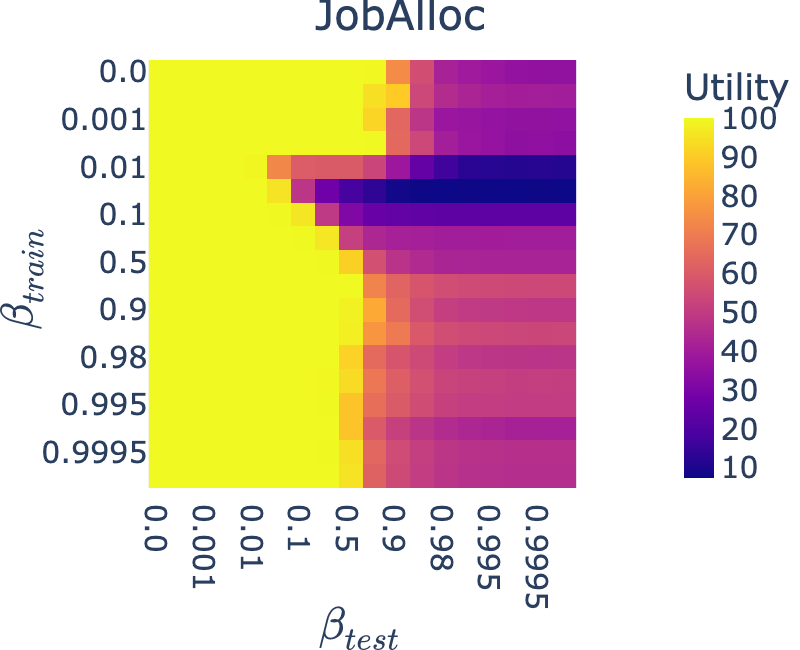}
            \includegraphics[width=0.48\linewidth]{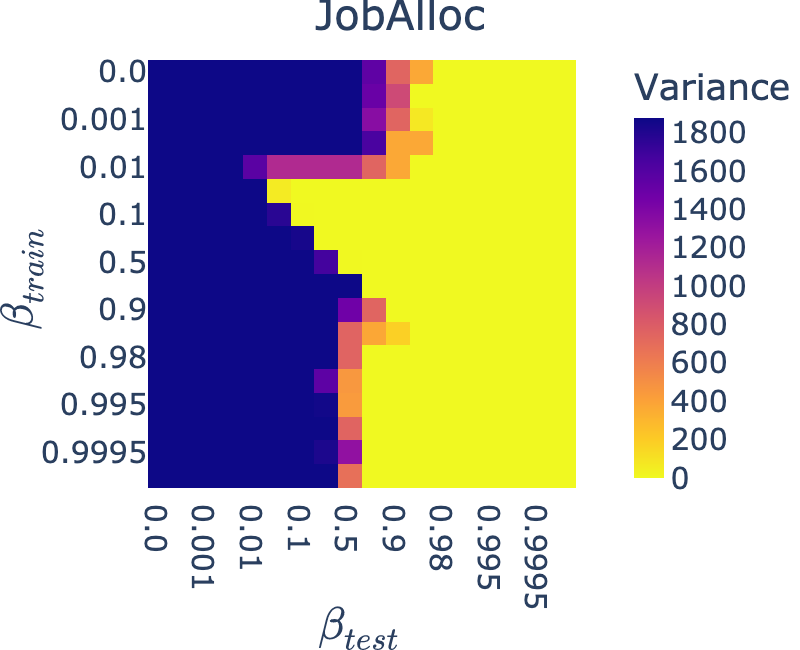}
            \\
            \includegraphics[width=0.48\linewidth]{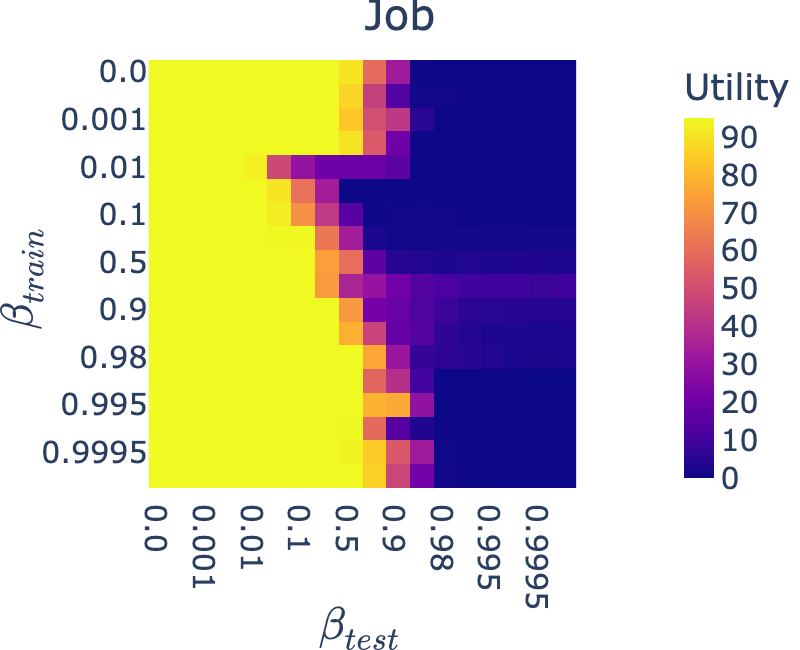}
            \includegraphics[width=0.48\linewidth]{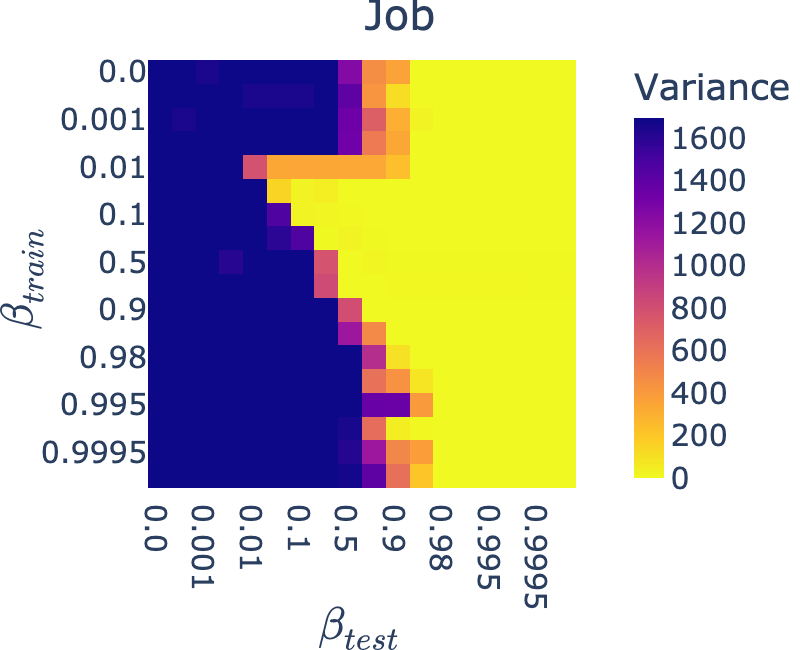}
            \\
            \includegraphics[width=0.48\linewidth]{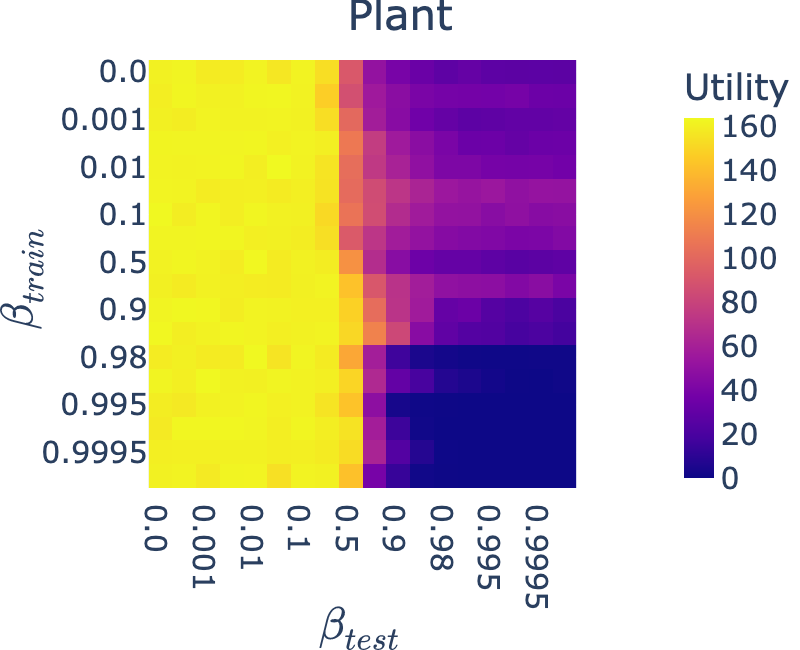}
            \includegraphics[width=0.48\linewidth]{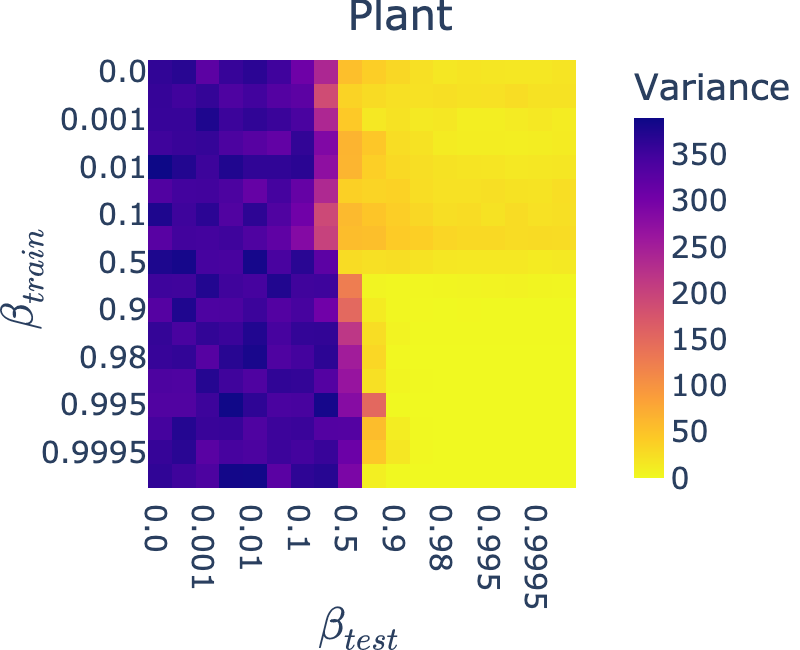}
            \\
            \includegraphics[width=0.48\linewidth]{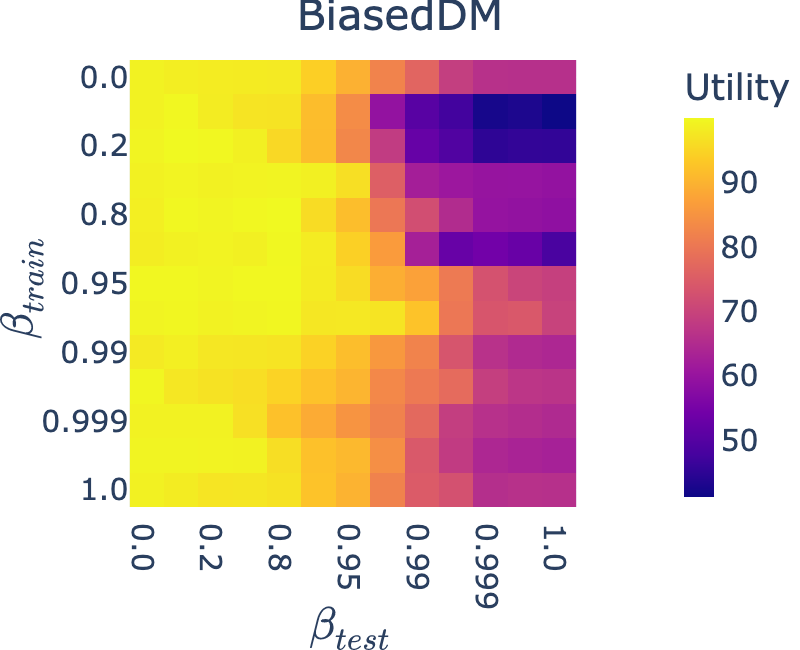}
            \includegraphics[width=0.48\linewidth]{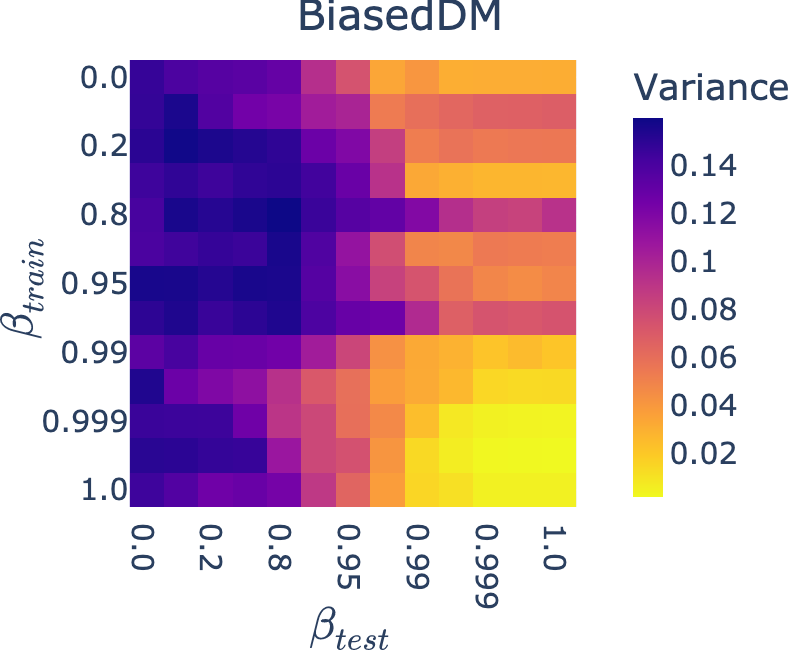}
        \end{minipage}
        \label{fig:all_results_gen_FO}
    }
    \caption{Evaluation of Split Optimization (left) and Fair-Only Optimization (right) models trained on $\beta_{train}$ (for variance) and evaluated on $\beta_{test}$ across different environments. Brighter colors indicate better outcomes.}
    \label{fig:combined_results}
\end{figure*}

Here we provide additional results for our methods, including providing confidence intervals for our main results, additional results for generalization, and more comparison to baselines. 

\subsection{Confidence Intervals for Results}
We provide confidence intervals for system utility and fairness separately as we vary $\beta$, as plotting this on a Pareto plot (as in the main results) would be hard to read (Figure~\ref{fig:error-bars}). An interesting thing to note here is that we see a phase lag between when variance starts to reduce, and when utility starts to drop. This shows there is a range of solutions where fairness can be improved without significantly harming the utility.

\subsection{Approximating the Pareto Front}
We also show generalization results by generating approximate Pareto fronts for all methods (Figure~\ref{fig:approx_pareto_fronts})   based on a limited set of $\beta_{train}$ models, showing that the SO and FO methods generalize very well even without training for all intermediate $\beta$ values. One interesting thing to note here is that FO falls under the Pareto front with intermediate $\beta$ values, showing how a tuned utility model also helps in guiding agents towards better decisions.

\subsection{Comparison of Selected Models}
Figure~\ref{fig:all_barplots} shows the performance of selected DECAF models trained on variance when evaluated on different metrics on all environments, compared to the baselines. We can see DECAF models perform much better on all metrics. For BiasedDM, all methods were able to converge to approximately the same fairest policy, so the differences in the figure are small. The variance for SO might appear large, but we point out that this variance is scaled 10000 times, and the actual values are all very close to zero.

\subsection{Generalization Heatmaps}
We also provide the generalization results for varying $\beta_{test}$ for both Split (Figure~\ref{fig:all_results_gen_SO}) and Fair Only (Figure~\ref{fig:all_results_gen_FO}) models for all environments. We note that the general trends noted in the main experimental results hold, with each model able to improve fairness as $\beta_{test}$ is increased. Further, in all cases, FO maximizes utility at $\beta_{test}=0$, and has a sharper transition between utility-maximizing and fairness-maximizing behavior.

\end{document}